\documentclass{article}

     \PassOptionsToPackage{numbers, compress}{natbib}
\usepackage[preprint]{neurips_2026}

\usepackage[utf8]{inputenc} 
\usepackage[T1]{fontenc}    
\usepackage{hyperref}       
\usepackage{url}            
\usepackage{booktabs}       
\usepackage{amsfonts}       
\usepackage{nicefrac}       
\usepackage{microtype}      
\usepackage{xcolor}         

\usepackage{amsmath}
\usepackage[capitalize,noabbrev]{cleveref}
\usepackage{amssymb}
\usepackage{mathtools}
\usepackage{amsthm}
\usepackage{dsfont}
\theoremstyle{plain}
\usepackage{enumitem}

\newtheorem{theorem}{Theorem}[section]
\newtheorem*{theorem*}{Theorem}

\theoremstyle{definition}

\newtheorem{assumption}[theorem]{Assumption}
\theoremstyle{remark}

\crefname{assumption}{Assumption}{Assumptions}
\Crefname{assumption}{Assumption}{Assumptions}

\newcommand{\theoremGD}{Consider the gradient flow dynamics $\dot{\bU} = -\nabla_{\bU} L$ and $\dot{\bV} = -\nabla_{\bV} L$ on the loss $L(\bU, \bV)$ starting from infinitesimal initialization. Then:
\begin{enumerate}
    \item Fixed Points: Any point $(\bU, \bV)$ such that $\bV \bU = \sum_{k=1}^r s_k q_k r_k^\top$ (for $r \le D$) with $u_i\in\text{span}\{r_k\}_{k=1}^r$, $v_i\in \text{span}\{q_k\}_{k=1}^r$ is a fixed point of the dynamics.
    \item Sequential Learning: The $k$-th singular value $\sigma_k(t)$ of the product matrix $\bV \bU$ evolves according to a sigmoid function, with a convergence time $\propto s_k^{-1}$. Consequently, for any $k$ such that $s_k > s_{k+1}$, there exists a time interval when the $k$-th mode is fully learned ($\sigma_k \approx s_k$) while the $(k+1)$-th mode is negligible ($\sigma_{k+1} \approx 0$).
\end{enumerate}}

\newcommand{\theoremSpecGD}{Observe the spectral gradient flow $\dot{\bU} = - \text{orth}(\nabla_\bU L)$ and $\dot{\bV} = - \text{orth}(\nabla_\bV L)$ where $\text{orth}(\cdot)$ is the orthogonalization of a matrix. Starting from infinitesimal initialization, we have:
    \begin{enumerate}
    \item Solution Trajectory: The learning trajectory of $\bW=\bV \bU$ will sequentially pass through all ${\bW_r := \sum_{k=1}^{r} s_{r+1} q_k r_k^\top + \sum_{k=r+1}^D s_kq_k r_k^\top}$ for $r\le D$ (from higher to lower) such that ${s_{r}>s_{r+1}}$.
    \item Equal Learning: To get from one $\bW_r$ to the next, $\bW$ learns singular spaces of $\Sigma_{yx}$ in the same time, at the same rate, until saturation. Furthermore, the evolution of $k$-th singular value $\sigma_k(t)$ of $\bW$ follows a quadratic curve, leading to convergence time $\propto \sqrt{s_k}$.
\end{enumerate}}

\definecolor{dgreen}{rgb}{0.0, 0.52, 0.34}

\newcommand{\bV}{{\boldsymbol{V}}}
\newcommand{\bL}{\boldsymbol{\Lambda}}
\newcommand{\bLh}{\hat{\boldsymbol{\Lambda}}}
\newcommand{\bK}{{\boldsymbol{K}}}
\newcommand{\bQ}{{\boldsymbol{Q}}}
\newcommand{\be}{{\boldsymbol{e}}}
\newcommand{\bg}{{\boldsymbol{g}}}
\newcommand{\bx}{{\boldsymbol{x}}}

\newcommand{\Rl}{\mathbb{R}}
\newcommand{\bk}{{\boldsymbol{k}}}
\newcommand{\bq}{{\boldsymbol{q}}}
\newcommand{\bw}{{\boldsymbol{w}}}
\newcommand{\bX}{{\boldsymbol{X}}}
\newcommand{\R}{\mathbb{R}}
\newcommand{\bI}{{\boldsymbol{I}}}
\newcommand{\bW}{{\boldsymbol{W}}}
\newcommand{\bU}{{\boldsymbol{U}}}
\newcommand{\bS}{{\boldsymbol{S}}}
\newcommand{\bG}{{\boldsymbol{G}}}
\newcommand{\bC}{{\boldsymbol{C}}}

\usepackage{algorithm}
\usepackage{algorithmic}
\usepackage{wrapfig}

\title{To Use or not to Use Muon: \\How Simplicity Bias in Optimizers Matters}

\author{%
  Sara Dragutinović$^{1,}$\thanks{Correspondence to: \texttt{sara.dragutinovic@nyu.edu}} \quad
  Yedi Zhang$^2$ \quad
  Rajesh Ranganath$^1$ \\[2ex]
  $^1$Courant Institute School of Mathematics, Computing, and Data Science, New York University \\
  $^2$Gatsby Computational Neuroscience Unit, University College London
}

\begin{document}

\maketitle

\begin{abstract}
  While Adam has long been the ubiquitous default optimizer for deep neural networks, Muon has recently seen rapid adoption due to its superior training speed. Although much of the literature focuses on validating the benefits of Muon, our work investigates the potential downsides of the mechanism driving this speedup. On the theoretical front, we analyze the learning dynamics of simplified Muon on deep linear networks and linear attention. Our analysis reveals that Muon gains speed by avoiding saddle points, but does so at the expense of the simplicity bias characteristic of Gradient Descent (GD), where the complexity of the functional solution learned grows sequentially. Experiments demonstrate the consequences of losing the simplicity bias, showing that Muon struggles to uncover common underlying structure across tasks and may be prone to fitting spurious features. More broadly, this paper serves as a reminder that faster optimization is rarely a free lunch; improvements in optimization can come at the cost of changes in the inductive biases that shape generalization.
\end{abstract}

\section{Introduction and Background}
Employing deep learning in practice usually involves careful decisions on the model architecture and data (e.g. deciding on augmentations or feature engineering), while the choice of the optimizer typically relies on established defaults such as Adam, or more recently, Muon. This paper aims to spotlight the impact of optimizer choice: different optimizers take different optimization paths in the loss landscape, each with their own inductive biases. Despite this, prior works primarily focus on the benefit of faster optimization \citep{wen2026fantastic}. There remains a lack of theory or intuition on how the learning trajectory taken impacts functional properties of the learned solution.

Specifically, we focus on the Muon optimizer \citep{jordan2024muon} which has attracted a large audience due to its performance on the NanoGPT Speedrun competition \citep{modded_nanogpt_2024}, learning much faster (in wallclock time) than other optimizers. It has since been adopted as one of the new defaults, for example, in training NanoChat \citep{nanochat} and DeepSeek-V4 \citep{deepseekai2026deepseekv4}, with speculations it may be used for language models in industry as well \citep{su2025isotropic}. However, we still lack an understanding of Muon's inductive biases, the trajectories it takes through the loss landscape to be fast, and the implications of these dynamics on the solution Muon-optimized model converges to.

The goal of this paper, apart from emphasizing asking those questions, is to make initial steps in answering them. Prior work has shown that Muon \textit{does} learn all the features, e.g. much faster than GD in a setting where their presence is largely imbalanced \citep{vasudeva2026how, wang2026muon}. In our work, we focus on the question \textit{how} does Muon learn all these features. In a non-linear toy setting, we show that a good mental model for feature learning dynamics is the singular value evolution in deep linear networks: Muon learns all features simultaneously at an equal rate, while SGD learns them one by one, in the order of descending strength. In other words, Muon does \textit{not} possess this type of dynamical simplicity bias observed in SGD \citep{saxe2013exact, gidel2019implicit, nakkiran2019sgd, refinetti2023neural, zhang2026saddletosaddle}. Investigating the consequences of this, we find that Muon's equal learning of features can result in failure to capture generalizable structures. Interested in whether these findings hold more broadly, for more modern architectures, we explore the transformer model. Our theoretical analysis of learning dynamics on linear attention indicates the lack of yet another type of simplicity bias in Muon: Muon learns heads all together, rather than one-by-one. Our experiments on a two-layer softmax transformer reveal downsides of Muon-learned solution. 
By presenting these findings, we aim to encourage a more nuanced evaluation of Muon and other optimizers that looks beyond its empirical speed to also consider the functional properties of the learned solution and how they generalize. 
We summarize our contributions as follows:
\begin{itemize}[leftmargin=20pt]
    \item Motivated by the theory of GD- \citep{saxe2013exact} and simplified Muon- \citep{vasudeva2026how} optimized deep linear networks (Section \ref{sec: dlnn}), we show on a non-linear toy example that feature learning with both optimizers behaves similarly to the dynamics of singular values described by the theory (Section \ref{sec: non-lin}). Importantly, both theory and empirics indicate that Muon loses the simplicity bias of GD, and learns all features equally, at the same time.
    \item Contrary to prior beliefs \cite{shah2020pitfalls, vasudevarich, vasudeva2026how}, we show that optimizers with simplicity bias can surpass the ones without in a setting with spurious features. The bigger picture is that the winner is not determined only by the optimizer, but by data properties as well (Section \ref{sec: spur}).
    \item Wanting to understand better the potential downsides of Muon's equal feature learning, in a toy version of a multimodal task, we show that simplified Muon tends to \emph{not} learn representations that generalize (Section \ref{sec: routing}).
    \item Shifting to more modern architectures, we theoretically analyze the learning dynamics of GD and simplified Muon on multi-head linear attention. Theory states that while GD learns the heads one-by-one, simplified Muon learns all the heads simultaneously (Section \ref{sec: transformer_theory}).
    \item Learning all heads at the same time could impact model's generalization, as related concepts may be learned by separate heads instead of having a common representation. We demonstrate this with a two-layer softmax transformer, which, when trained with Muon, does not extrapolate to unseen (but related) data, unlike its SGD-trained counterpart (Section \ref{sec: transformer}).
\end{itemize}

The key takeaway here is that dynamical simplicity bias acts as an implicit curriculum, forcing learning one by one, with the order of learning given by geometric properties of the model and data. Muon abandons this curriculum to optimize across multiple axes at once; while this rapidly minimizes error, it prevents the model from stepping through simpler solutions that would have revealed the latent structures required for generalization---and this happens in practice!

\paragraph{Muon Optimizer and Spectral GD.}
We briefly introduce the optimization step of \textbf{Muon} (MomentUm Orthogonalized by Newton-Schulz), as introduced in \citet{jordan2024muon}. Unlike SGD and Adam, Muon takes advantage of the usual setup in deep learning, where most of the parameters are organized in two-dimensional weight matrices. Let ${\bW}^{(t)}\in \Rl^{d_{\text{out}}\times d_{\text{in}}}$ be a weight matrix after optimization step $t$, and ${\bG}^{(t)}$ the usual GD update of the weights with momentum $\mu$, i.e. ${{\bG}^{(t)} = \mu {\bG}^{(t-1)}+\nabla_{\bW} \mathcal L({\bW}^{(t)})}$. The update of Muon at time $t+1$ and learning rate $\eta$ aims to be $${\bW}^{(t+1)} = {\bW}^{(t)} - \eta {\bU}^{(t)}{\bV}^{(t)\top},$$
where ${\bG}^{(t)} = {\bU}^{(t)}{\bS}^{(t)}{\bV}^{(t)\top}$ is the Singular Value Decomposition (SVD) of ${\bG}^{(t)}$\footnote{The SVD performed is compact, meaning the diagonal matrix ${\bS}^{(t)}$ is of dimension $r\times r$ ($r$ being the rank of ${\bG}^{(t)}$) and contains only positive singular values.}. In words, Muon has the goal to set all non-zero singular values of the update matrix to be equal to 1. This process is termed \textit{orthogonalization} of ${\bG}^{(t)}$. However, computing the SVD at every step is computationally prohibitive for deep learning. To address this, Muon employs Newton-Schulz iterations to efficiently approximate orthogonalization \citep{doi:10.1137/0707031, 59139e86-5891-3aa5-b8e6-33b9eeb82afa}. 

We also introduce \textbf{Spectral GD} \citep{pmlr-v38-carlson15, 7347351}, a simplification of Muon. Compared to Muon, Spectral GD (i) uses exact SVD to orthogonalize the update, instead of approximate SVD with Newton-Schulz iterations; (ii) doesn't use momentum, but orthogonalizes $\nabla_{\bW} L({\bW}^{(t)})$ directly. Spectral GD retains important properties of Muon (i.e., orthogonalization) and is more mathematically tractable. We study Spectral GD for the theory and switch back to the full Muon for experiments.
\section{Deep Linear Networks: Spectral GD Loses Simplicity Bias}
\label{sec: dlnn}
\begin{figure*}
    \centering
    \includegraphics[width=0.85\linewidth]{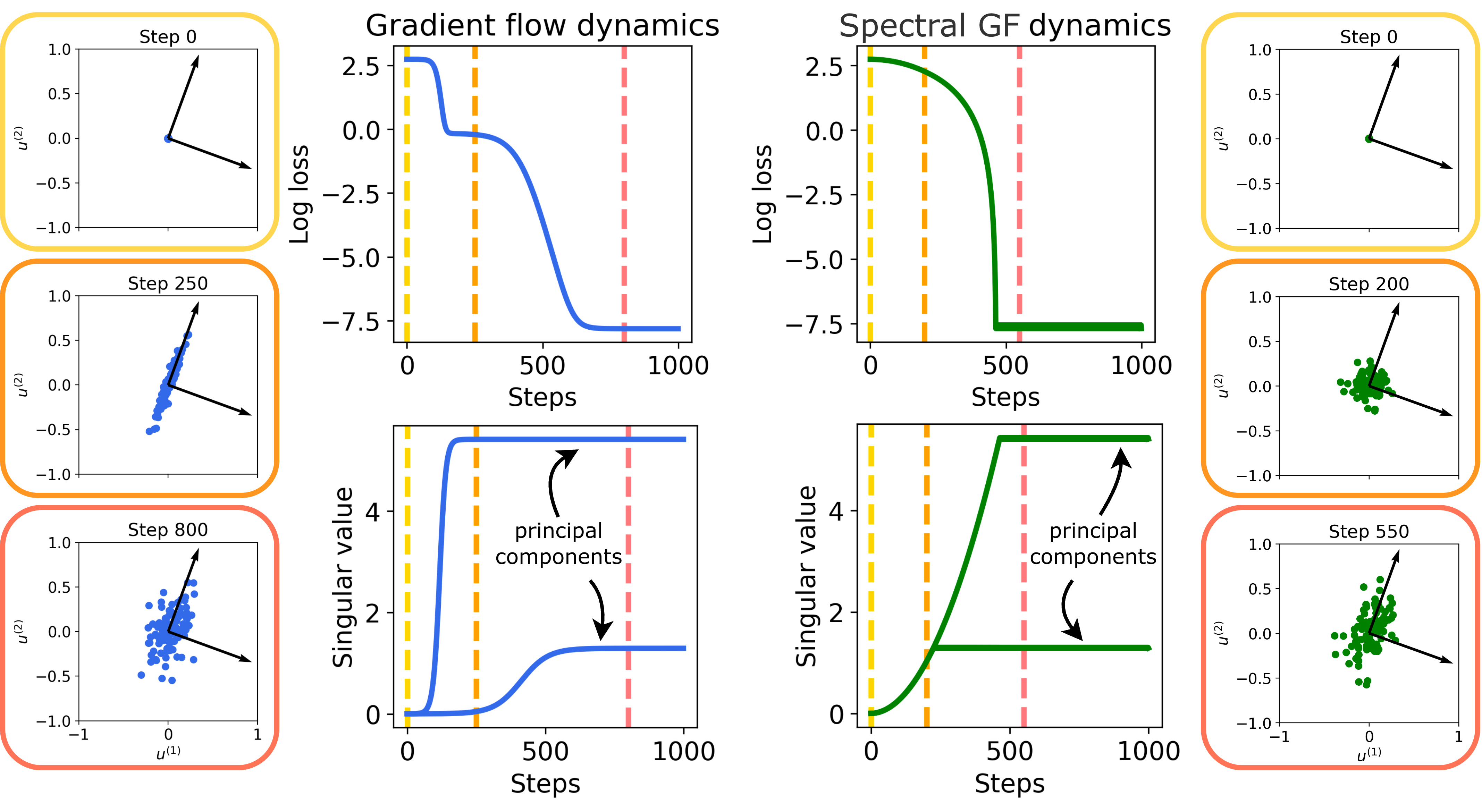}
    \caption{Illustration of the theory presented in Section \ref{sec: dlnn}, for gradient flow (\textit{Left}) and Spectral GD (\textit{Right}). The top row depicts the loss curve, the bottom one the evolution of singular values of the neural network. We also show the evolution of each row $u_i$ of the first layer matrix $U\in \Rl^{H\times d_{in}}$ for $H=100$ hidden neurons (dots in the circled plots). Black arrows are right singular vectors of the input-output data correlation matrix (see Appendix \ref{appendix: theory} for better understanding). We observe that for GD, first singular vector is fully learned first, and only then the second one is learned. On the other hand, Spectral GD learns both of them in the same time, and after it saturates on the smaller one, then $u_i$s progress only in the direction of the larger one. This simulation is closely following theorems \ref{theorem: gd}, \ref{theorem: spectral gd} and proofs in Appendix \ref{appendix: theory}, also containing the experimental setup details.}
    \label{fig:theory}
\end{figure*}
Our journey towards understanding Muon's biases starts with theoretical analysis on deep linear networks trained with mean squared error loss. We informally present it here, attribute the results to \citet{vasudeva2026how}, and compare to classical results for GD \citep{saxe2013exact}. Formal statements and proofs are provided in Appendix \ref{appendix: theory}.

\begin{theorem}[Gradient flow dynamics---Informal]
\label{theorem: gd}
Gradient flow dynamics in two-layer deep linear network exhibits gradual increase in complexity: starting from small initialization, principal components are learned one by one, from larger singular values $s$ to smaller, sequentially increasing the rank of the learned solution. For each principal component to be learned, a saddle point needs to be escaped, taking $\propto 1/s$ time.
\end{theorem}

\begin{theorem}[Spectral gradient flow dynamics---Informal]
\label{theorem: spectral gd}
Spectral gradient flow learns all principal components with the same speed, at the same time. Hence the ones with the smaller singular values $s$ are fully learned first, in time $\propto \sqrt{s}$.
\end{theorem}
An illustration of \cref{theorem: gd,theorem: spectral gd} is provided in \cref{fig:theory}.
The geometrical reason GD is slower is that it encounters many saddle points along the optimization trajectory. In contrast, Spectral GD doesn't pass through these (or other) saddles, but also it doesn't learn the solution gradually---\textit{it loses this type of simplicity bias}. 
Spectral GD can learn modalities appearing with different frequencies in the training set faster and more uniformly than GD, as is pointed out by prior and concurrent work of \citet{su2025isotropic, vasudeva2026how, wang2026muon}. On the other hand, such a learning is more \textit{greedy}, in the sense that it learns all the different singular vectors simultaneously. An orthogonal example from \citet{puli2023don} suggests that faster learning can trade-off for the price of the solution relying more on spurious features---motivating our experiments later. Simplicity bias is often necessary as it forms an implicit curriculum \cite{liu26curriculum} in learning.

\section{Qualitative Extension to Non-Linear Setting}
\label{sec: non-lin}
Equipped with theoretical understanding of simplified Muon in a linear sandbox, we question whether that theory can help us understand more complex settings. To this end, we employ the full Muon to optimize the cross-entropy loss of a non-linear model, CNN, with stochastic batches. For all the details and further experiments, see Appendix \ref{appendix:features}, and for Q\&A Appendix \ref{appendix: qa}.

\paragraph{Task.}
We are interested in model's feature learning through training, analogous to principal components learning in the theory above. Namely, each image in our training data has a \textbf{digit feature} (handwritten MNIST digit \citep{6296535}) , and a \textbf{pixel feature}. All handwritten images of the digit $i$ have the pixel feature---a pixel with intensity 1---at the same location, different one than digits $j\ne i$. This makes both features predictive of the actual label. To evaluate the models, we use 4 different sets, shown in Figure \ref{fig:features}a):
\begin{itemize}
    \item \textit{Test} evaluator comes from the same distribution as the training set, where images contain both the handwritten digit and the pixel feature corresponding to the label; 
    \item \textit{Digit} evaluator contains only the digit feature---this is just the MNIST test set;
    \item \textit{Misaligned; label=digit} evaluator has, on an image of digit $i$, the pixel feature of the label $(i-1)\pmod {10}$. The correct label corresponds to the digit feature;
    \item \textit{Misaligned; label=pixel} has the same inputs as the previous one, but the correct label now corresponds to the pixel feature.
\end{itemize}
\begin{figure}
    \centering
    \includegraphics[width=\linewidth]{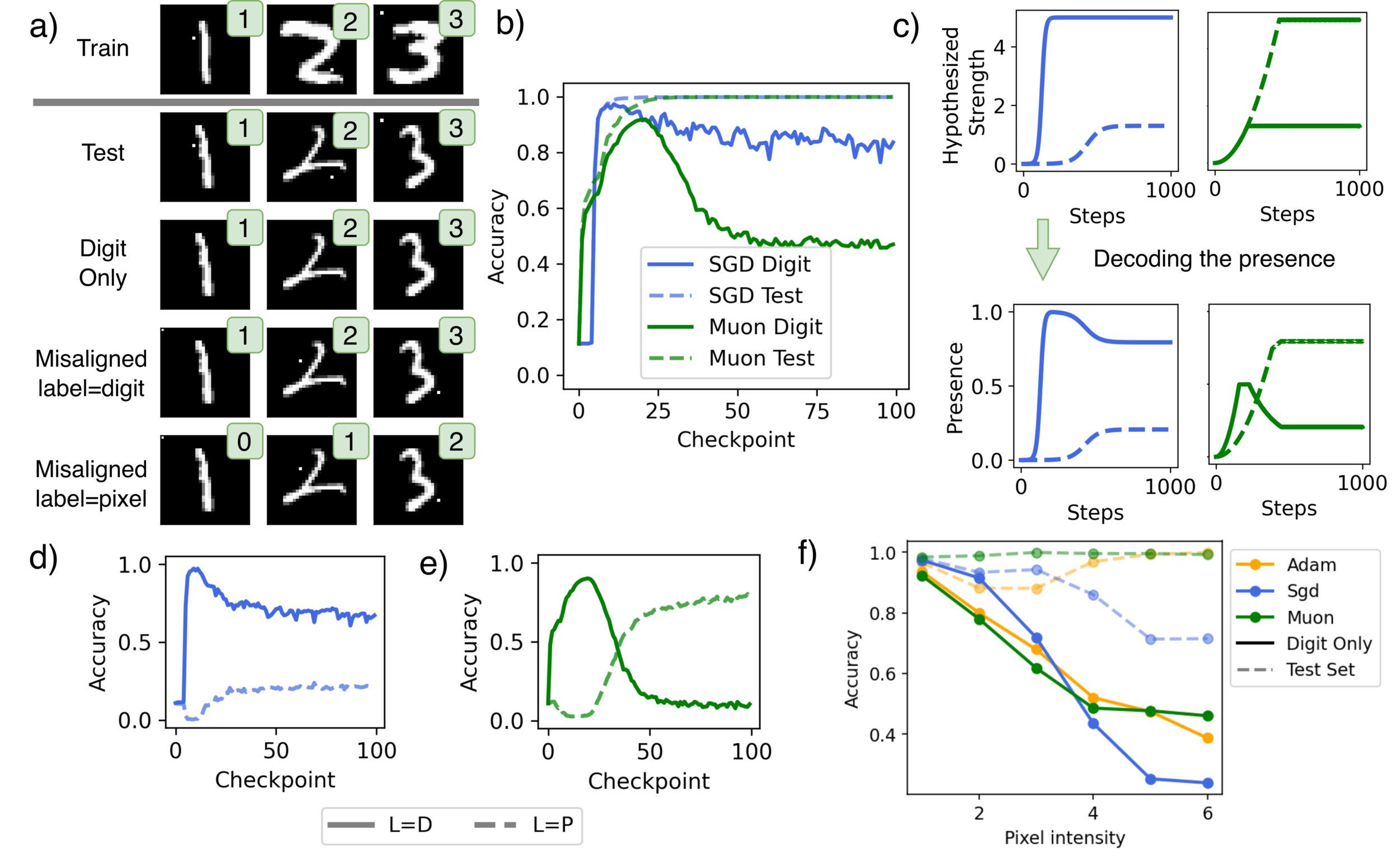}
    \caption{a) Training samples, along with different evaluators in the two feature learning setup. b) Accuracy on Test and Digit only evaluators. c) Converting hypothesized strengths to presence of each feature with (\textit{Left}) GD and (\textit{Right}) Muon. d) Accuracy of SGD on Misaligned evaluators with label=digit (L=D) and label=pixel (L=P). e) Accuracy of Muon on the same sets. We notice how these competing accuracies qualitatively follow the decoded hypothesized presence of the two features. f) Peak accuracy on the Digit only evaluator and Test accuracy (dashed) at the point that peak is achieved. The accuracies are plotted as we increase the intensity of the spurious pixel. The result illustrates that there is not a fixed optimizer better with spurious features; data properties also play a key role in determining the winner.}
    \label{fig:features}
\end{figure}
\paragraph{Experiments.} Taking a look at the accuracies on Test and Digit evaluators in Figure \ref{fig:features}b), we see that SGD pursues a solution relying only on the digit feature for much longer: the two curves separate only at the peak accuracy on Digit set. Muon's curves however start separating earlier, indicating that at the peak accuracy on Digit set, the model is already non-negligibly relying on the pixel feature as well. This reminds us of the theory: in the early training, GD has learned only the largest principal component, where simplified Muon has learned all components in the same time. 

To connect with the theory more explicitly, we need an analogue of singular values. For this purpose, we define a concept of \textbf{strength} of each of the two features, imitating the two singular values. Imagine that the strength of each feature through training can be modeled by curves similar to singular value evolution for each optimizer. Empirically it's not clear how to extract feature strength from the model, but we can measure which feature the model is relying on more---we call the relevant quantity feature \textbf{presence}. We can extract feature presence by measuring the accuracy on Misaligned; label=digit and Misaligned; label=pixel evaluators---a model cannot give the correct answer to both, so there is competition. 

In theory, to go from hypothesized strength $s_i(t)$ of feature $i$ to its hypothesized presence $p_i(t)$ at time $t$, we use the following formula: ${p_i(t) = s_i(t)\min(\frac 1 {\max_t{s_i}(t)}, \frac 1 {s_1(t)+s_2(t)})}$. The first term in the minimum just serves as normalization, not letting the feature be present more than $100\%$. The second term is responsible for competition, making sure that ${p_1(t)+p_2(t)\le 1}$. This means that once both features have large strength, their presence will be proportional to their strength. We plot the decoded feature presence from our hypothesized feature strengths in Figure \ref{fig:features}c).

Comparing hypothesized presence with its empirical counterpart on the Misaligned evaluators for both SGD and Muon (Figure \ref{fig:features}d, e) reveals a striking qualitative similarity. This suggests that \textbf{evolution of the singular values is a good mental model for feature learning with both optimizers}.

Importantly, we also see that for SGD the digit feature has larger strength: the final accuracy on label=digit is larger than on label=pixel (see Appendix \ref{appendix:qa_dp} for a longer run). The contrary holds for Muon model, which tells us that pixel feature is more dominant in the final solution learned by Muon. Crucially, this demonstrates that \textbf{the relative strength of learned features is intrinsically tied to the choice of optimizer, even when the models are exposed to identical training data}. This may be highly relevant when dealing with spurious features, as we investigate next.

\subsection{How simplicity bias interacts with spurious features}
\label{sec: spur}
Spurious features are defined as `simpler' features in the training data that could be predictive, but do not generalize to other settings \citep{puliout}. An example of a spurious feature could be the scanner imprint on a medical scan, an image background in classification of birds, or the pixel feature in our MNIST setup \citep{pulinuisances}. Prior work has claimed that simplicity bias of SGD hurts in these settings compared to Adam and Muon, as it leads to simpler solutions that rely on spurious features \citep{shah2020pitfalls, vasudevarich, vasudeva2026how}. 

Our results in Figure \ref{fig:features}b) indicate differently.  First, they show that SGD can outperform Muon (and Adam, see Appendix \ref{appendix:qa_dp}), measured by the peak accuracy on Digit or Misaligned; label=digit sets. Furthermore, the situation can easily change: when varying the intensity of the spurious pixel in Figure \ref{fig:features}f), observe the transition from SGD being better to Muon as the intensity increases. The hypothesis from the theory explains this: while the intensity is low enough, the digit feature is more dominant and SGD's simplicity bias is helpful. The situation changes for larger intensities, when the simplicity bias focuses on the more dominant, spurious feature, hurting the performance of SGD; Muon learns both features at the same time which in this situation helps. This shows that \textbf{optimizing with Muon is a double-edged sword: while learning is more uniform across the features, task and the data are determining whether that is a good thing or not}.

\section{Spectral GD does not Learn Generalizable Representations}
\label{sec: routing}
The better understanding of Muon's equal feature learning leads us to question the potential downsides compared to GD. Our testbed for that purpose is a toy task which can be solved in multiple ways, but only one---the one understanding the common underlying structure---achieves good out of distribution generalization. In that case, we say that the model has learned the \textit{generalizable representations}.

\paragraph{Setup.} We use the `routing' task setup from \citet{saxe2022neural}, a simplification of a multimodal model. The architecture and task details are explained in Figure \ref{fig:nrr}a,b). While all different input and output sources use their own linear encoders and decoders respectively, the common hidden layer is used by all. During training,
each sample $(x, y)$ will have its input source (i.e. which domain $x$ came from) and output source (domain of $y$), and only corresponding layers are used and trained. One could imagine this setting being a toy simplification for multi-modal networks, where perhaps different sources can include text, images, sound etc. Underlying task is very simple: the goal is to learn the mapping from 4 numbers $\{1,2,3,4\}$ to 4 vectors, as shown in Figure \ref{fig:nrr}b).
However, each input source has a different vector representations for $\{1,2,3,4\}$: for each input source, we generate four $4$-dimensional orthonormal vectors, encoding the four numbers. For simplicity, we don't transform the output vectors based on output domain (but note that still, each output domain has their own linear decoder). We use MSE loss between predicted output vector and the true one (details in Appendix \ref{appendix: nrr}).

The catch of the task is that not all input-output pairs of sources are seen in training. Specifically, we only see samples $(x, y)$ where $x$ came from input source $j$, and $y$ is from output domain $j$ or $j+1 ~(M+1:=1)$, for all $j=1,...,M$. The common linear layer is of dimension $64\times 64$, large enough to memorize all the pairings seen in training, even without realizing that the underlying task has a much simpler solution.
 \begin{figure}
    \centering
    \includegraphics[width=0.9\linewidth]{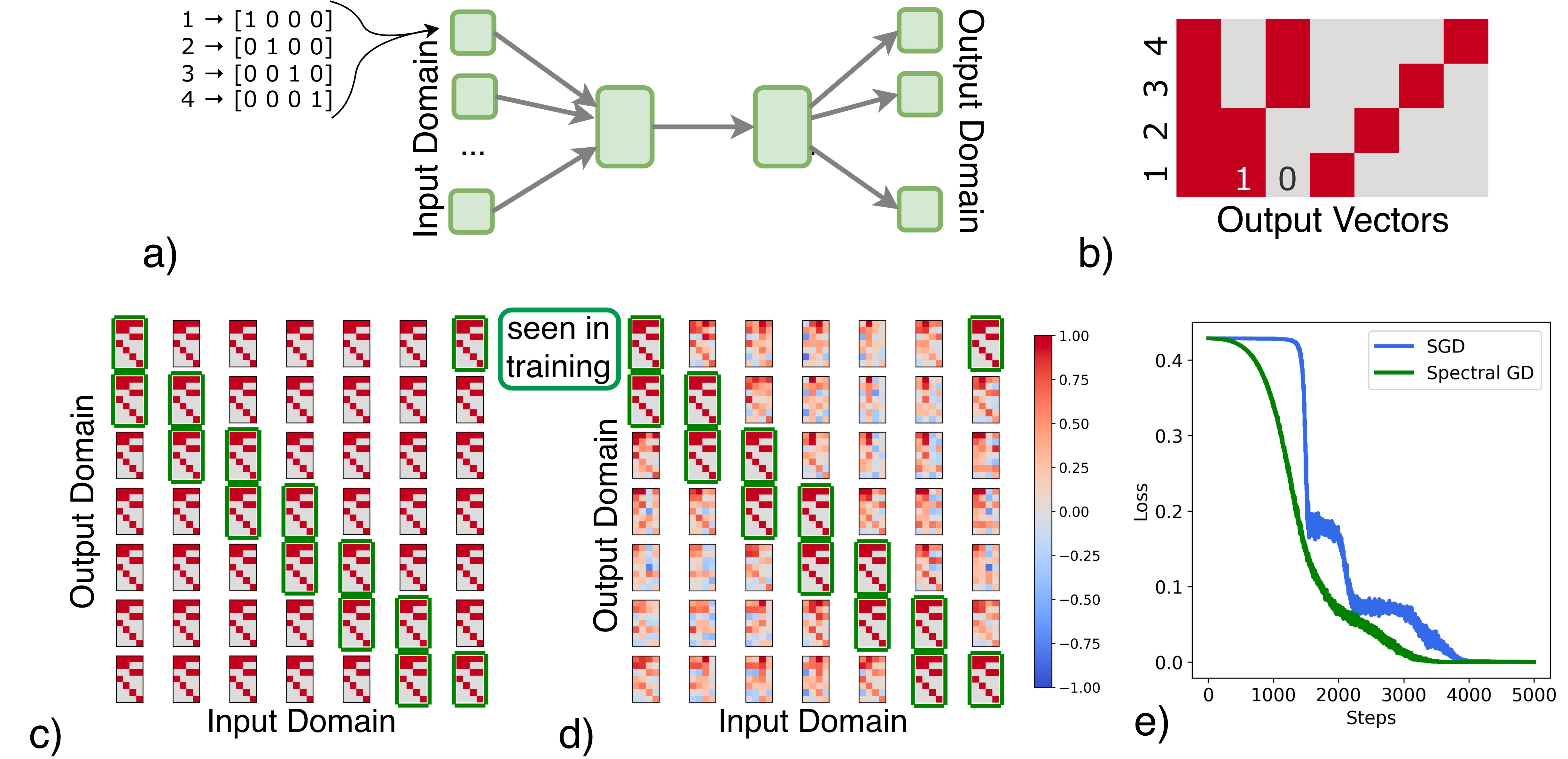}
    \caption{a) The neural network used to solve the task, where each of the gray arrows is a linear layer, with no nonlinearities in between. There is $M=7$ input and output domains. Each input domain has its own, fixed 4 orthonormal vectors to represent $\{1,2,3,4\}$. b) The underlying task we're learning: mapping each number in $\{1,2,3,4\}$ to the output vector shown. Results after training in the `routing' setup with c) SGD and d) Spectral GD, together with e) training loss curves. We plot the function the models learned (4 different column vectors represent the image of $\{1,2,3,4\}$) for all the different input-output pairs of sources. Circled in green are the pairs seen during training.}
    \label{fig:nrr}
\end{figure}
\paragraph{Experiments.} The empirical results are captured in Figure \ref{fig:nrr}c,d), where we show the function learned across all input-output pairs of sources. Training with SGD does result in learning the shared representations across all inputs, because even the input-output pairs that have never been seen during training do end up producing the correct underlying function. This generalization, however, doesn't happen when optimizing with Spectral GD, in the stochastic setting with batches. While Spectral GD does learn to solve the task for input-output pairs seen during training, it does so by memorizing each input-output pair seen, as can be inferred from the poor generalization on input-output pairs not present in training. It is important to note that both algorithms achieve perfect convergence (Figure \ref{fig:nrr}e), thus the hope of Spectral GD recovering the underlying shared representations with more data is gone. Further confirmation comes from the hidden layer’s spectral properties: whereas the SGD solution converges to an approximate low-rank structure (rank $=4$, equaling the dimension of the underlying task), the Spectral GD solution exhibits a significantly higher effective rank with a heavy-tailed spectrum---a clear signature of memorization rather than structural learning.

This example demonstrates that the removal of simplicity bias, while yielding faster convergence, imposes a tangible cost on the structural quality of the learned solution. \textbf{By prioritizing rapid loss reduction, the optimizer bypasses the discovery of common latent structures}, converging instead to a less generalizable solution. Next, we show this behavior can be found in attention.

\section{Theory on Spectral GD with Multi-Head Linear Attention}
\label{sec: transformer_theory}
Intrigued by the results in the simpler architectures, we question whether these effects persist in more modern architectures, particularly in transformers. Although recent work has studied the saddle points and simplicity bias in GD-optimized transformers \citep{boix2023transformers,edelman24iclmarkov,rende24transformer,geshkovski24metastability,zhang2025training,varre25ngram,yuksel2026incremental,saha2026staged}, no such analysis exists for Muon. To fill in this gap, we theoretically study the linear self-attention model \citep{oswald23GD} trained on an in-context linear regression task \citep{garg22function} and build on the training dynamics analysis of \citet{zhang2025training}. The mathematical tractability of this setting comes from the scalar valued output, thus our choice is to study this specific setting.
\begin{figure}
    \centering
    \includegraphics[width=0.9\linewidth]{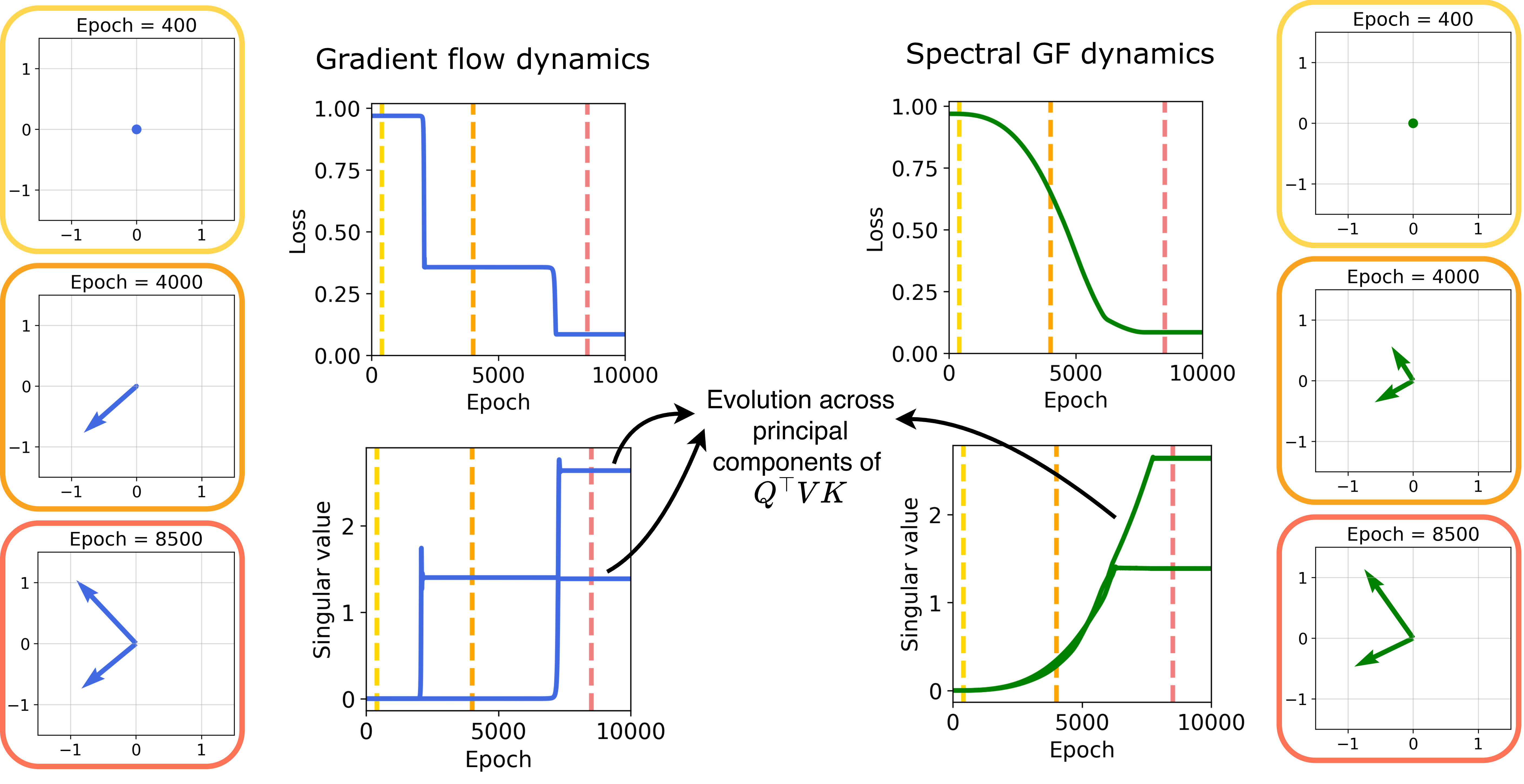}
    \caption{Learning dynamics for GF (Left) and Spectral GF (Right). Top row shows the loss curves, and bottom one the growth of singular values of $\bQ^\top \bV \bK$ for each of the $D=2$ principal components. Framed in orange are the snapshots of $\bk_1$ and $\bk_2$ through training. This simulation follows the theory, illustrating the head-by-head learning with GF and all heads in the same time with Spectral GF.}
    \label{fig:tr_dyn}
\end{figure}

\paragraph{Task.}
Each training sample $\mu=1,\dots, P$ is generated by drawing a $D$-dimensional task vector $\bw_\mu \sim \mathcal N(0, \bI_D)$, and drawing $(N+1)$ $D$-dimensional $\bx$ tokens, $\bx_{\mu, n}, \bx_{\mu, q}\in \mathbb R^{D} \, (n=1,\dots, N)$, independently from $\mathcal N(\mathbf 0, \bL)$. The input sequence is constructed as
\begin{align}
\bX_\mu = \begin{bmatrix}
\bx_{\mu,1} & \cdots & \bx_{\mu,N} & \bx_{\mu,q}  \\
y_{\mu,1} & \cdots & y_{\mu,N} & 0
\end{bmatrix}
= \begin{bmatrix}
\bx_{\mu,1} & \cdots & \bx_{\mu,N} & \bx_{\mu,q}  \\
\bw_\mu^\top \bx_{\mu,1} & \cdots & \bw_\mu^\top \bx_{\mu,N} & 0
\end{bmatrix} 
\in\mathbb R^{(D+1) \times (N+1)} .
\end{align}
The first $N$ columns of $\bX_\mu$ are the context and $\bx_{\mu, q}$ is the query. The target output is $y_q=\bw_\mu^\top \bx_{\mu,q}$, which is the task vector of this sequence applied to the query. This task is called in-context linear regression because the model needs to linearly regress the task vector in context using $N$ pairs of $(\bx_{\mu,n},y_{\mu,n})$ and then apply it to the query.

\paragraph{Model.}
We analyze one-layer linear attention with $H$ heads and rank-one key and query weights in each head.
Following prior work \citep{zhang24jmlr,zhang2025training,cengiz25asymptotic}, we study linear attention in which weight entries irrelevant to this task are initialized to zero. In this regime, the prediction of the model for $y_{\mu, q}$ given input $\bX_\mu$ is $\hat y_{\mu, q} = \sum_{i=1}^H v_i \left(\frac 1 N\sum_{n=1}^N y_{\mu, n} \bx_{\mu, n}^\top \right) \bk_i \bq_i^\top \bx_{\mu, q}$, where the value, key, and query weights in the $i$-th head are denoted as $v_i\in \mathbb R$, and $\bk_i, \bq_i\in \mathbb R^D$. Writing these in matrix form gives $\bK, \bQ\in\Rl^{H\times D}$ and a diagonal matrix $\bV\in\Rl^{H\times H}$ with i-th entry on the diagonal $v_i$.
We train the model with MSE loss, $L = \mathbb E[(y_{\mu, q}-\hat y_{\mu, q})^2]/2$. We denote the in-context covariance of $\bx_{\mu,n}$ as $\hat\bL_\mu = \sum_{n=1}^N \bx_{\mu,n} \bx_{\mu,n}^\top /N$. To facilitate comparison between GD and Spectral GD, we rewrite the gradient flow dynamics from \citet{zhang2025training} in matrix form:
\begin{align*}
    \frac{d\bV}{dt} &=\texttt{diag} \left( \bK \left( \bL^2 - \mathbb{E} \left[ \bLh^2 \right] \bK^\top \bV \bQ \bL \right) \bQ^\top \right)\\
    \frac{d\bK}{dt} &= \bV \bQ \left( \bL^2 - \bL \bQ^\top \bV \bK \mathbb{E} \left[ \bLh^2 \right] \right)\\
    \frac{d\bQ}{dt} &= \bV \bK \left( \bL^2 - \mathbb{E} \left[ \bLh^2 \right] \bK^\top \bV \bQ \bL \right).
\end{align*}
Under small initialization, gradient flow exhibits saddle-to-saddle dynamics \cite{zhang2025training}, learning one head at a time (Figure \ref{fig:tr_dyn} Left). For completeness, we provide the proof in Appendix \ref{appendix: gf lin attn theorem}, under the additional assumptions stated below. Under the same assumptions, we show that Spectral GF's dynamics is characterized by the following theorem, with the formal statement and proof in Appendix \ref{appendix: spectral gd lin attn}.

\textbf{Assumptions.} We assume the attention is parameterized with weight matrices $\bQ, \bK, \bV$, and Spectral GD updates are applied to those matrices. We also assume joint diagonalizability of $\bQ,\bK, \bV$ and $\bL$ at small initialization, meaning the left principal components of $\bK, \bQ$ are the canonical basis vectors in $\Rl^H$, and the right ones are the eigenvectors of $\bL$. We assume $H=D$, and for $H>D$ simulations are in Appendix \ref{appendix: lin attn sim}. 

\begin{theorem}[Spectral gradient flow on linear attention---Informal]
\label{theorem: specgd_attn}
Unlike gradient flow, in the setting described, Spectral gradient flow doesn't pass through saddle points and hence doesn't learn one head at a time. Instead, all heads are being learned at the same time and rate, linearly per component $\bQ, \bK$ and $\bV$. This leads to each head $i$ learning the eigenspace of $\bL$ with eigenvalue $\lambda_i$ in time $\propto \lambda_i^{-1/3}$.\footnote{This is for the case $N\to \infty$, for $N<\infty$ see the formal version.} For a simulation, see Figure \cref{fig:tr_dyn} Right.
\end{theorem}

With GF, simplicity bias in deep linear networks took a different form: weights of all hidden neurons evolve, but they grow along one principal component direction at a time. In linear attention, weights in one head grow at a time, while the others remain near zero. We have shown that simplified Muon removes both types of simplicity bias. In the case of GD, \citet{zhang2025training} showed that with heads of higher rank $R>1$, once a head is activated (meaning its weights reach $O(1)$), learning $\bQ$ and $\bK$ empirically follows the dynamics of a deep linear network, as indeed these two matrices are multiplied in the attention. For Spectral GF, this would imply that transformer architecture is prone to losing both types of simplicity biases. Further analysis of learning dynamics of attention matrices with Muon seems promising, and can help us justify major decisions in large scale models. One example is applying Muon to all heads concatenated, as is usually done, or to each separately, as found empirically beneficial by \citet{zeng2026glm}.

\section{A Potential Downside of Using Muon on Transformers}
\label{sec: transformer}
Thinking about generalization, it is important that (i) a single concept is learned by a single head and (ii) closely related concepts are learned by the same head. Muon learning all heads simultaneously may take away the implicit curriculum of SGD, resulting in ungeneralizable representations. In this section, we provide empirical evidence for this.

\paragraph{Model.} We train a two-layer causal transformer on MSE loss, with 2 softmax heads in each layer. Our inputs are one hot vectors and we don't use token embedding and unembedding. We use RoPE positional encoding. For more details, see Appendix \ref{appendix: transformer}, and for ablations see Q\&A in \ref{appendix: qa_tr}.

\paragraph{Task.} The task itself should be simple, featuring shared underlying structure and easy evaluation to test whether the model has learned it. Satisfying all this, we introduce the \textit{cycle task}: the model is trained on next token prediction to learn the numbers from 1 to $N$ in cyclical order, with different skips. Namely, to draw a sample, first draw its skip $s$ from a subset of $\{1,2,...,N-1\}$ uniformly, then the first number in the sequence $i$ also uniformly, independently of the skip. The sample is given by the sequence $(i, (i+s)\pmod N, (i+2s)\pmod N,...)$ with the length $10N$, and $0\equiv N$. Figure \ref{fig:tr}a) contains a couple of examples for the setting $N=4$. The catch this time is that the model sees only a strict subset of skips during training, and the goal is to see whether it learns to generalize through evaluating on unseen skips.

\paragraph{Experiments.} We evaluate the SGD and Muon in the $N=8$ setting, training on the skip set $\{1, 2, 3, -1, -3\}$. Although the models do not encounter skip $-2$ during training, they are exposed to two forward-backward pairs: $(1, -1)$ and $(3, -3)$. We investigate whether observing these symmetries, alongside the presence of skip $2$, provides sufficient signal for the models to extrapolate to the unseen skip $-2$. The results, presented in Figures \ref{fig:tr}b, c), demonstrate that the SGD-optimized transformer learns skip $-2$ nearly as well as skip $2$. In contrast, the Muon-optimized model exhibits poor performance on this extrapolated task, suggesting that \textbf{Muon struggles to learn unified representations that generalize} across different skips. The heads both optimizers learned, shown in Figure \ref{fig:tr}d, e), exhibit the theory-suggested differences: while SGD's solution is effectively low-rank, with higher contribution from one head than the other, Muon's solution uses both heads more equally. This failure to abstract shared rules poses a risk for tasks like mathematics, where the goal is not to memorize specific derivation steps, but to acquire reusable techniques applicable to unseen problems. Concurrent work of \citet{liu2026optimizer} empirically shows that Muon trained and finetuned GPT-2 memorizes more than Adam and SGD counterparts.
\begin{figure}
    \centering
    \includegraphics[width=\linewidth]{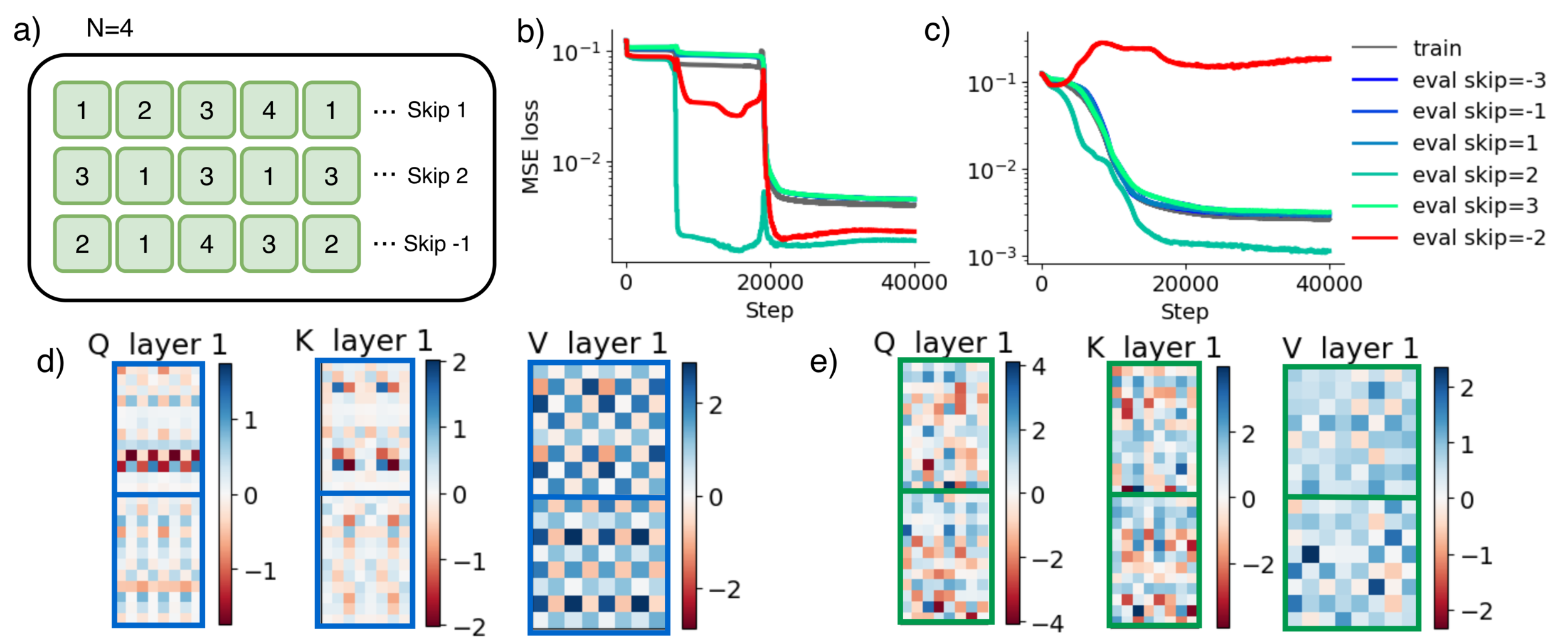}
    \caption{a) Three example sequences with different skips from our task, for $N=4$. Loss on the in- and out-of-distribution data in our transformer setup, when optimized by \textbf{b) SGD} and \textbf{c) Muon}. The model has been trained on all skips shown, apart from skip=-2 (red). d) Final weights of the SGD trained transformer. e) Muon trained instead. The two heads are separated by the framing. Notice how the SGD algorithm picked up the underlying cyclic structure and generalized to unseen skip, where Muon learned all skips more separately, thus not generalizing.}
    \label{fig:tr}
\end{figure}

\section{Discussion}
The recent advancements in deep learning optimizers design have moved the field. While optimization algorithms are usually derived from theoretical principles, it is often unclear how these specific update rules shape the inductive biases and ultimate behavior of the trained deep learning models. We focused on the Muon optimizer \citep{jordan2024muon}, and explored the consequences of its implementational choices. Theoretically, we can explain why it is faster than GD alternatives, but has the trade-off of losing a simplicity bias. Motivated again by the theory, we came up with a couple of examples where SGD outperforms Muon, an outcome rarely seen in prior work. First, in an example with a spurious feature, we showed how the theory applies, and witnessed Muon being worse due to learning all features simultaneously. Contextualizing this within prior work on the pitfalls of simplicity bias, we demonstrated that there is no universally optimal optimizer with spurious features; instead, the interaction between the optimization algorithm and data properties dictates performance. We also evaluated Muon on underlying tasks with shared structure across settings and noticed how Muon fails to learn unified representations. This failure happened even for transformers leading to Muon generalizing worse than SGD. These results show the importance of understanding optimizers' biases, not just comparing them by their speed.

\textbf{Bridging our understanding of learning dynamics closer to practice.} Orthogonal to exploring the downsides of Muon, this work contributes to narrowing the gap in theoretical understanding of learning dynamics in more practical settings. Prior theoretical work analyzed the training dynamics across a wide range of architectures \citep{zhang2026saddletosaddle}, while all assuming gradient descent as an optimization procedure. We explain how the dynamics change when switching from gradient descent to the Muon simplification in the case of linear attention, revealing yet another loss of simplicity in its trajectory. While hinting at their mathematical tractability, we leave exploring other Muon optimized architectures, such as CNNs and ReLU neural networks, for future work.

\paragraph{Limitations and Future Work.} Our work, trying to keep closer to theoretically tractable settings, did not answer the questions about Muon optimizer at larger scale. If anything, we hope to see a shift in these questions, from `How fast is it?' to the ones about the underlying functions learned. One could argue that our example tasks were artificially hiding parts of the data, but can we be sure that is also not the case in LLM pretraining sets? Another open question is whether introducing curriculum into training with Muon can mitigate its ignorance of shared representations. In our experiments in Appendix \ref{appendix: qa_tr}, data curriculum did not help, however, imposing a curriculum to the order of weight learning helped. Also, an elephant in the room is the fact that we don't yet know how to train at large scale with SGD, due to many encountered instabilities. One would potentially need a better way of dealing with saddle points if interested in traversing the same trajectory GD takes. More work on overcoming this issue can be promising \citep{zhao2025deconstructing}.

Altogether, we hope this paper sparks new ideas in the field of deep learning optimizers and that in the future it is more common to analyze the inductive biases of these optimizers and the implications of such. Deciding on what makes a good optimizer is a hard question, and once we have the answer (or an idea for a high quality answer), perhaps it will become easier to act upon it.

\subsubsection*{Acknowledgments}
The authors would like to thank Jatin Prakash for useful discussions and feedback. SD and RR were partly supported by the NIH/NHLBI Award R01HL148248, NSF Award 1922658 NRT-HDR: FUTURE Foundations, Translation, and Responsibility for Data Science, NSF CAREER Award 2145542, ONR N00014-23-1-2634, NIH R01CA296388, NSF 2404476, Optum, and Apple. They were also supported by IITP with a grant funded by the MSIT of the Republic of Korea in connection with the Global AI Frontier Lab International Collaborative Research. YZ was supported by the Gatsby Charitable Foundation (GAT4058).

\newpage
\bibliographystyle{plainnat}
\bibliography{ref}


\newpage
\appendix

 \section{Extended Related Work}
We mostly base our theory in Section \ref{sec: dlnn} for GD on \citet{zhang2026saddletosaddle} and for Spectral GD on \citet{vasudeva2026how}. The theory for GD in Section \ref{sec: transformer_theory} is based on \citet{zhang2025training}. 

\paragraph{Saddle-to-saddle dynamics and simplicity bias.}Saddle-to-saddle dynamics \citep{jacot2021saddle,berthier23diagonal,abbe2023sgd, zhang2026saddletosaddle} is a phenomenon observed across architectures \citep{maennel2018gradient,teney2022evading,boix2023transformers}, causing the learning to happen by gradually increasing the complexity of the solution. This gives rise to a particular type of simplicity bias \citep{arpit2017closer, gidel2019implicit, rahaman2019spectral, hu2020surprising}.
While complex tasks may ultimately require complex representations, Occam's razor remains a vital inductive principle for robust learning and generalization. Because of this, the question of simplicity bias in deep architectures has been broadly studied, both theoretically \citep{valle2018deep, arora2019implicit} and empirically \citep{nakkiran2019sgd, shah2020pitfalls, teney2022evading}. Most of the theoretical papers used gradient descent for the optimizer \citep{nakkiran2019sgd, gidel2019implicit, refinetti2023neural}. This motivated us to think about the biases of other optimizers, especially Muon, as we can show it avoids these saddles. 

\paragraph{Understanding Muon.}
From the theory side, \citet{su2025isotropic} analyses Muon optimizer on isotropic curvature model. \citet{shen2025convergence} provide a convergence rate analysis of (simplified) Muon, in comparison with GD.

Most of the empirical papers focus on the benefits of Muon. In \citet{vasudeva2026how} they compare optimizers on imbalanced classes setting using MNIST-CIFAR dataset and Tiny Stories for language generation. 
\citet{wang2026muon} show that in a transformer, MLP and OV-circuit benefit the most from Muon (versus Adam), and also, motivated by associative memory model, show that Muon outperforms Adam on a memorization task (QA dataset about biographical information) where each dataset entry appears with different frequency. They state '\textit{Muon consistently yields more isotropic weight matrices with broadly distributed spectral energy than Adam, both throughout training and across random initializations, thereby supporting richer feature representations}'. In Section \ref{sec: routing}, we gave an example where 'richer' in their sense (i.e. learning everything there is to learn) is actually bad. Both of these works emphasize how Muon is better in the setting of imbalanced data, learning all modalities more evenly.

\citet{vasudeva2026how} also provides a comparison of the optimizers in a setting with spurious correlations, however, different than ours. In the data setting where 99\% of samples contains a dominant spurious feature, they show that while all optimizers learn the task perfectly, Muon learns it the fastest. They conclude: '\textit{This suggests that Muon’s spectral design promotes balanced learning of these components, leading to superior generalization.}'. While the premise is true, we showed that whether Muon generalizes better depends on the task---in the setting where we want to learn the more dominant features of the data, SGD wins.

Existing literature has largely focused on validating the efficiency benefits of Muon. To complement these findings, we investigated specific scenarios where Muon's inductive biases lead to suboptimal performance. By exploring these trade-offs, we aim to foster a more comprehensive understanding of the optimizer and motivate future algorithmic improvements.

\section{Theory on Deep Linear Networks}
\label{appendix: theory}
Here we formalize the theory on training dynamics for deep linear networks for both GD and Spectral GD. The well-studied framework of deep linear networks \citep{fukumizu98batch, saxe2013exact, saxe2019mathematical, lampinen2018analytic, ji2019gradient, jacot2021saddle, ziyin2022exact} is particularly interesting for understanding the non-linear learning dynamics in deep models. Gradient flow results can be found in the prior work studying the simplicity bias of GD \citep{saxe2013exact, nakkiran2019sgd, gidel2019implicit, jacot2021saddle, zhang2026saddletosaddle}. For  Spectral GD, we largely focus on the results in  \citet{vasudeva2026how}. The theory explains why Spectral GD learns faster, but also reveals the cost of it: loss of the simplicity bias of GD.

\paragraph{Setup.}
Let a forward pass of a two-layer deep linear network be $\hat y = \bV \bU x$, where $\bV\in\R^{d_{\text{out}}\times H}, \bU\in \R^{H\times d_{\text{in}}}$, $x\in \Rl^{d_{\text{in}}}$, and $H$ is the number of hidden neurons. Let the rows of $\bU$ be $u_i^\top$ and the columns of $\bV$ be $v_i$. We are given $n$ input-output pairs $\{(x_i, y_i)\}_{i=1}^n$. The goal is to minimize the MSE loss $L(\bU, \bV) = \frac 1 {2n} \sum_i \|\hat y_i - y_i\|^2$. Here we use full-batch gradient flow, starting from infinitesimal initializations $\bU_0, \bV_0$. Dynamics dependence on the dataset is fully captured by these statistics: 
\begin{align}
\Sigma_{xx} = \frac 1 n \sum_{i=1}^n x_i x_i^\top ,\quad
\Sigma_{yx} = \frac 1 n \sum_{i=1}^n y_i x_i^\top .
\end{align}

\begin{assumption}
We assume $\Sigma_{xx}=\bI$ and the joint diagonalizability of $\Sigma_{yx}$, $\bU_0$ and $\bV_0$.
\end{assumption}
The identity covariance assumption is required for solvable gradient flow dynamics \citep{fukumizu98batch,saxe2013exact}. For relaxing $\Sigma_{xx}=\bI$, we refer to \citet{gidel2019implicit, advani20highd, vasudeva2026how, watanabe2026impact}.
The joint diagonalizability assumption is justified by the fact that under small initialization, the principal components of $U, V$ do align with the ones of $\Sigma_{yx}$ early on in the training, a phenomenon known as the silent alignment \citep{atanasov22silent}.

\paragraph{Gradient flow solution.}
The gradients of $L$ with respect to $\bU$ and $\bV$ are given by 
\begin{align}
\nabla_{\bU} L = {\bV}^\top (\bV\bU\Sigma_{xx} - \Sigma_{yx}), \nabla_{\bV} L = (\bV \bU \Sigma_{xx} - \Sigma_{yx})\bU^\top.
\end{align}
Let
$\Sigma_{yx} = \sum_{k=1}^D s_k q_k r_k^\top$ be the SVD of $\Sigma_{yx}$ where $s_k$ are positive singular values. Let them be ordered $s_1\ge...\ge s_D$.

\begin{theorem}[Gradient Flow Dynamics]
\label{thm:gf}
\theoremGD
\end{theorem}

The proof, although standard, can be found in Appendix \ref{appendix:proof_gf}. Combining the two results leads to the explanation of saddle-to-saddle dynamics from small initialization. At the start of training, all singular values of the solution $\bW=\bV \bU$ are effectively 0. The rank of $\bW$ then increases gradually, first learning the highest singular values $s_k$. At the point rank has gone up to $r$ (for all $r\le D$), $\bV \bU$ effectively satisfies the conditions of 1, and the learning trajectory passes close to the corresponding saddle point. This is reflected in a loss plateau, and 2. predicts the time of escaping this saddle. In case of a singular value $s_k$ having multiplicity $m>1$, the rank at that point will increase by $m$ and the whole singular subspace will be learned at the same time.
\paragraph{Spectral gradient flow solution.}Here we use the same notation and assumptions as in the previous case. For spectral gradient descent, the dynamics is different. This time, all principal components are learned simultaneously, with the same speed, until each one saturates.

\begin{theorem}[Spectral Gradient Flow Dynamics]
\label{thm:spectral_gf}
\theoremSpecGD
\end{theorem}

We provide the proof in Appendix \ref{appendix:proof_spectral}, or refer to \citet{vasudeva2026how}, with slightly different assumptions. Consequently, training with Spectral GD also happens in phases: in each phase, singular values $\sigma_{r+1}(t) = s_{r+1},..., \sigma_D(t)=s_{D}$ are fully learned, while $\sigma_1(t), ..., \sigma_r(t)$ make their transition from $s_{r+1}$ to $s_r$, the next smallest singular value. The contrast with GD is the order of learning the singular values: while GD fully learns the higher ones first, Spectral GD learns all of them in the same time, leading to smaller ones being fully learned first. 
\paragraph{Experiments.}We both validate and illustrate the theory in Figure \ref{fig:theory}, where we perform full dataset (Spectral) GD on standard Gaussian data $x_i\in \R^{d_{\text{in}}}$, and $y_i\in \R^{d_{\text{out}}}$ computed as a noised linear function of $x_i$. In this usual linear regression setting, we sample the regression weights also as a standard normal. The weights $\bU, \bV$ are initialized also as Gaussian, but with small variance ($0.01$). We see that the evolution closely follows the theory.

Empirically, because we're performing a discretization of the dynamics of `Spectral gradient flow', and due to using a spectral method, it does happen that after the last $D-r$ principal components have been learned, the update matrix in the dynamics above is not exactly rank $r$ but also have some noise, and is usually full rank. In that case, Spectral GD does fit this noise, causing a highly oscillatory path in the loss landscape (see Figure \ref{fig:osc}), and the oscillations don't go away even if we include the momentum, as per Muon optimization (Figure \ref{fig:oscmom}). Although in a simple setting of deep linear networks the model always converges to the correct solution (as it is a stable minimum), this suggest that in more complicated settings Spectral GD and Muon may be less robust when using higher learning rate. Potentially, this observation could motivate future improvements of spectral optimizers.

\subsection{Proof of Theorem \ref{thm:gf}}
\label{appendix:proof_gf}

\begin{theorem*}
\theoremGD
\end{theorem*}

\begin{proof}
\textbf{Fixed Points.} 
The gradients for the deep linear network under $\Sigma_{xx}=\bI$ are:
\begin{align*}
    \nabla_\bU L &= \bV^\top(\bV \bU - \Sigma_{yx}), \\
    \nabla_\bV L &= (\bV\bU - \Sigma_{yx})\bU^\top.
\end{align*}
Let the candidate solution be $\bW_r = \bV \bU = \sum_{k=1}^r s_k q_k r_k^\top$. Substituting this into the residual term $(\bV \bU - \Sigma_{yx})$ yields:
\begin{equation*}
    \bW_r - \Sigma_{yx} = \sum_{k=1}^r s_k q_k r_k^\top - \sum_{k=1}^D s_k q_k r_k^\top = -\sum_{k=r+1}^D s_k q_k r_k^\top.
\end{equation*}
By the theorem assumption, $\text{row}(U) \in \text{span}\{r_k\}_{k=1}^r$ and $\text{col}(V) \in \text{span}\{q_k\}_{k=1}^r$, hence:
\begin{align*}
    \nabla_\bU L &= -\bV^\top \left( \sum_{k=r+1}^D s_k q_k r_k^\top \right) = 0, \\
    \nabla_\bV L &= -\left( \sum_{k=r+1}^D s_k q_k r_k^\top \right) \bU^\top = 0.
\end{align*}
Thus, $(\bU, \bV)$ is a fixed point.

\textbf{Sequential Learning.} This is a standard result found in many deep linear networks work \citep{saxe2013exact, gidel2019implicit, li2020towards, jacot2021saddle, zhang2026saddletosaddle}. For completeness, we provide a proof here. 
We analyze the evolution of the product matrix $\bW = \bV \bU$. We assume the standard ``balanced" condition $\bV^\top \bV = \bU \bU^\top$, which is an invariant/conservation law of gradient flow if initialized infinitesimally, as we assume.
The dynamics of the product matrix $\bW$ are given by:
\begin{equation*}
    \dot{\bW} = \dot{\bV}\bU + \bV\dot{\bU} = (\Sigma_{yx} - \bW) \bU^\top \bU + \bV \bV^\top (\Sigma_{yx}-\bW).
\end{equation*}
Under the balanced condition, we have $\bU^\top \bU = (\bW^\top \bW)^{1/2}$ and $\bV \bV^\top = (\bW \bW^\top)^{1/2}$. Substituting this into the dynamics yields:
\begin{equation*}
\label{eq:eq2}
    \dot{\bW} = -(\bW - \Sigma_{yx})(\bW^\top \bW)^{1/2} - (\bW \bW^\top)^{1/2}(\bW - \Sigma_{yx}).
\end{equation*}
The fact that $\bU(0), \bV(0)$ and $\Sigma_{yx}$ are jointly diagonalizable implies that $\bW_0 = \bV(0)\bU(0)$ and $\Sigma_{yx}$ are as well. If $\Sigma_{yx} = QSR^T$ is an SVD decomposition, then SVDs of matrices involved are: $\bW_0 \bW_0^\top = Q D_1 Q^T$,  $\bW_0^\top \bW_0 = R D_2 R^\top$ and $(\Sigma_{yx} -\bW_0) = QD_3 R^\top$.  As taking the square root doesn't change left and right principal components, we infer that the full gradient update in Equation \ref{eq:eq2} is also aligned with the principal components of $\Sigma_{yx}$. Hence the alignment of principal components of $\bW$ and $\Sigma_{yx}$ from the beginning stays true throughout training.

Let $\bW(t) = \sum_{k=1}^D \sigma_k(t) q_k r_k^\top$. Substituting this spectral decomposition into the matrix ODE decouples the system into independent scalar differential equations for each singular value $\sigma_k(t)$:
\begin{equation*}
    \dot{\sigma}_k(t) = -(\sigma_k - s_k)\sigma_k - \sigma_k(\sigma_k - s_k) = 2 \sigma_k(t) (s_k - \sigma_k(t)).
\end{equation*}
This is the logistic differential equation. The solution is:
\begin{equation*}
    \sigma_k(t) = \frac{s_k}{1 + \left(\frac{s_k}{\sigma_k(0)} - 1\right)e^{-2 s_k t}}.
\end{equation*}
This solution shows that $\sigma_k(t)$ follows a sigmoidal trajectory, saturating to $s_k$. Consequently, the time taken for $s_k$ to be learned is of order $1/s_k$, implying that modes with larger singular values are learned significantly faster, establishing the sequential learning property.
\end{proof}

\subsection{Proof of Theorem \ref{thm:spectral_gf}}\label{appendix:proof_spectral}
\begin{theorem*}
\theoremSpecGD
\end{theorem*}

\begin{proof}
\textbf{Solution Trajectory.} The proof, with slightly different assumptions, can be found in \citet{vasudeva2026how}. Nonetheless, we follow our setup and provide it here for completeness.

We assumed joint diagonalizability of $U(0) = \sum_{k=1}^D \sigma_k^U(0) z_k r_k^\top$, $V(0) = \sum_{k=1}^D \sigma_k^V(0) q_k z_k^\top$ and $\Sigma_{yx} = \sum_{k=1}^D s_k q_k r_k^\top$, which is justified by the infinitesimal initialization. As we will see, once $U(t), V(t)$ are jointly diagonalizable in this way, they stay so for all $t'\ge t$). Suppose for some $r\le D, t_0\ge 0$, ${W (t_0)= \sum_{k=1}^{r} s_{r+1} q_k r_k^\top + \sum_{k=r+1}^D s_kq_k r_k^\top}$. Indeed, this is satisfied at initialization $t_0=0$, with $r=D$, once we denote $s_{D+1}=0$. If it holds at $t_0$, then 
\begin{align*}
    \nabla_\bU L(\bU(t_0)) &= \bV(t_0)^\top(\bW(t)-\Sigma_{yx}) \\
    &= \left(\sum_{k=1}^D \sigma_k^V (t_0) q_k z_k^\top\right)^\top \left(\sum_{k=1}^{r} (\sigma_{r+1} - \sigma_k)q_k r_k^\top \right)\\
    &=\sum_{k=1}^r \sigma_k^V (t_0)(\sigma_{r+1} -\sigma_k) z_k r_k^\top \\
    \nabla_\bV L(\bV(t_0)) &= (\bW(t_0)-\Sigma_{yx}) \bU(t_0)^\top\\
    &= \left(\sum_{k=1}^{r} (\sigma_{r+1} - \sigma_k)q_k r_k^\top \right) \left(\sum_{k=1}^D \sigma_k^U (t_0) z_k r_k^\top \right)^\top\\
    &=\sum_{k=1}^r \sigma_k^U (t_0)(\sigma_{r+1} -\sigma_k) q_k z_k^\top
\end{align*} 
From this, the orthogonalizations of the gradients are $\text{orth}(\nabla_\bU L(t_0))=-\sum_{k=1}^r z_k r_k^\top, \text{orth}(\nabla_\bV L(t_0))=-\sum_{k=1}^r q_k z_k^\top$. This means that change in $\bU(t), \bV(t)$ happens only in the first $r$ singular vectors, and the dynamics will look as a gradient flow $\dot{\bU}(t) = \sum_{i=1}^r z_k r_k^{\top}, \dot{\bV}(t) =  \sum_{k=1}^r q_k z_k^\top$.
This keeps the matrices jointly diagonalizable, and therefore we can observe the dynamics on the singular values, from initial conditions at $t_0$:
\begin{align*}
    \dot{\sigma_k^V}(t) &= \mathds 1_{\{\sigma_k^V(t)\sigma_k^U(t)<s_k\}}\\
    \dot{\sigma_k^U}(t) &= \mathds 1_{\{\sigma_k^V(t)\sigma_k^U(t)<s_k\}},
\end{align*}
noting that $\sigma_k(t) = \sigma_k^V(t)\sigma_k^U(t)$. By symmetry at infinitesimal initialization, all $\sigma_k^V(t)$ and $\sigma_k^U(t)$ will be the same for $k\le r$, and hence for all $k\le r$, $\sigma_k(t)$ will evolve the same, from $\sigma_k(t_0) = s_{r+1}$ to $\sigma_k(t_1)=s_r$. At time $t_1$, for all $m< k\le r$ with $m$ largest s.t. $s_{m+1}=s_r$, the growth of $\sigma_k(t)$ stops. Hence the solution $W$ passes through $W_{m}$.

\textbf{Equal learning.} Continuing with the notation above, if we start from $\sigma_k^V(0)=\sigma_k^U(0)=0$, the solution of the system is $\sigma_k^U(t)=\sigma_k^V(t)=\min(t, \sqrt{s_k})$, and hence $\sigma_k(t) = \min(t^2, s_k)$. This shows both the equal growth and the quadratic curve for learning singular values of $\Sigma_{yx}$. Consequently, the time to learn $\sigma_k(t) = s_k$ is $t\propto\sqrt{s_k}$.

\end{proof}

\section{Theory on Linear Attention}
\label{appendix:theory_attn}
Here we formalize the theory on linear attention from Section \ref{sec: transformer_theory}. Under the setup described in Section \ref{sec: transformer_theory}, \citet{zhang2025training} showed the training dynamics for each head $i=1,\cdots,H$ is given by
\begin{subequations}
\begin{align*}
\frac {dv_i}{dt} &= \boldsymbol{k}_i^\top \left( \boldsymbol{\Lambda}^2 - \mathbb{E} \left( \hat{\boldsymbol{\Lambda}}^2 \right) \sum_{i'=1}^H v_{i'} \boldsymbol{k}_{i'} \boldsymbol{q}_{i'}^\top \boldsymbol{\Lambda} \right) \boldsymbol{q}_i ,\\
\frac{d\boldsymbol{k}_i}{dt} &= v_i \left( \boldsymbol{\Lambda}^2 - \mathbb{E} \left( \hat{\boldsymbol{\Lambda}}^2 \right) \sum_{i'=1}^H v_{i'} \boldsymbol{k}_{i'} \boldsymbol{q}_{i'}^\top \boldsymbol{\Lambda} \right) \boldsymbol{q}_i ,\\
\frac {d\boldsymbol{q}_i}{dt} &= v_i \left( \boldsymbol{\Lambda}^2 - \boldsymbol{\Lambda} \sum_{i'=1}^H v_{i'} \bq_{i'} \bk_{i'}^\top \mathbb{E} \left( \hat{\boldsymbol{\Lambda}}^2 \right) \right) \boldsymbol{k}_i.
\end{align*}
\end{subequations}
where $\bLh$ is the in-context covariance of $\bx$ tokens and the expectation of $\bLh^2$ is given by
\begin{align*}
\mathbb{E}\left( \bLh^2 \right) = \bL^2 + \frac{\bL + \text{tr}(\bL)\boldsymbol{I}}{N} \bL .
\end{align*}
The reason we are able to write the dynamics compactly is that the loss is only computed for a scalar entry in the output sequence $\hat y_q$. This is made possible in the in-context linear regression setting, as value matrix of the $\bx$-entry is irrelevant for our prediction, and we only care about the $y$-entry $v_i$, a scalar that commutes with other data statistics matrices. 

As mentioned in the main text, we collect the relevant key and query vectors for each head $\bk_i, \bq_i\in \Rl^D$ in matrices $\bK, \bQ\in\Rl^{H\times D}$. Value scalars of each head $v_i$ are collected in the diagonal matrix $\bV\in \Rl^{H\times H}$. Using this notation allows us to rewrite the dynamics in matrix form:
\begin{subequations}  \label{eq:linattn-matrix}
\begin{align*}
\frac{d\bV}{dt} &= \texttt{diag}\left(\bK \left( \bL^2 - \mathbb{E} \left[ \bLh^2 \right] \bK^\top \bV \bQ \bL \right) \bQ^\top\right) ,
\\
\frac{d\bK}{dt} &= \bV \bQ \left( \bL^2 - \bL \bQ^\top \bV \bK \mathbb{E} \left[ \bLh^2 \right] \right) ,
\\
\frac{d\bQ}{dt} &= \bV \bK \left( \bL^2 - \mathbb{E} \left[ \bLh^2 \right] \bK^\top \bV \bQ \bL \right) ,
\end{align*}
\end{subequations}
where $\texttt{diag}(\mathbf A)$ denotes a diagonal matrix with entries on the diagonal being the same as the ones in $\mathbf A$. Writing the dynamics like this doesn't change the the trajectory of gradient flow solutions compared to using the full matrices $\bW_K, \bW_Q\in \Rl^{H\times (D+1)}$ and $\bW_V\in \Rl^{H\times (D+1)\times (D+1)}$. However, in case of Spectral GF, the reparameterization from $\bW_K, \bW_Q, \bW_V$ to $\bK, \bQ, \bV$ does change the dynamics. Nonetheless, we use the reparameterization as it makes the Spectral GF solution mathematically tractable. The key phenomena of all heads being learned at the same time obtained from the theory is still present without the reparameterization, as shown in Figure \ref{fig:attn_full_matrices} in Appendix \ref{appendix: lin attn sim}.

We denote the eigenvectors of $\bL$ as $\{\bg_j\}_{j=1}^D$ and the corresponding eigenvalues as $\{\lambda_j\}_{j=1}^D$
\begin{align*}
\bL &= \sum_{j=1}^{D} \lambda_j \bg_j \bg_j^\top.
\end{align*}
The matrix $\mathbb{E}\left[ \bLh^2 \right]$ has the same eigenvectors as $\bL$ and eigenvalues $\{\hat{\lambda}_j\}_{j=1}^D$
\begin{align}\label{eq:lambdas}
\mathbb{E}\left[ \bLh^2 \right] &= \sum_{j=1}^{D} \hat{\lambda}_j \bg_j \bg_j^\top ,
\quad \text{where } \;
\hat{\lambda}_j = \lambda_j^2 \left( 1 + \frac{1+\text{tr}(\bL)/\lambda_j}{N} \right) .
\end{align}

\begin{assumption}  \label{ass:linattn}
For both GF and Spectral GF, we assume infinitesimal initialization and joint diagonalizability on the initialization of $\bK, \bQ, \bV$ and $\bL$. That is
\begin{align}  \label{eq:linattn-init}
\bV(0) = \sum_{i=1}^{H} v_i(0) \be_i \be_i^\top, \quad 
\bK(0) = \sum_{j=1}^{D} \sigma_j^K(0) \be_j \bg_j^\top, \quad 
\bQ(0) = \sum_{j=1}^{D} \sigma_j^Q(0) \be_j \bg_j^\top,
\end{align}
with $v_i(0)=\sigma_i^K(0)=\sigma_i^Q(0)=\epsilon_i$ for all $i=1,..., H$, where $\epsilon_i$ is infinitesimal.
\end{assumption}
In the case of gradient flow, \citet{zhang2025training} explained with their ansatz that from arbitrary but small initialization, the head vectors $\bk_i, \bq_i$ will align with the eigenvectors $\bg_i$, hinting that the assumption may be redundant. Intuitively, this is because with gradient flow, the growth of $\bk_i$ along each of eigenvector $\bg_i$ will approximately follow the dynamics $\frac{d(\bk_i^\top \bg_i)}{dt} = \lambda_i^2 (\bk_i^\top \bg_i)^2$, exhibiting the fastest grow of $\bk_i$ along the eigenvector associated with largest eigenvalue $\lambda_i$ (and similarly for $\bq_i$). In case of Spectral GD, this doesn't hold, as growth along each eigenvalue is better modeled by $\frac{d(\bk_i^\top \bg_i)}{dt} = \mathds 1_{\bk_i^\top \bg_i<(\lambda_i)^{-1/3}}$. This makes the growth along each $\bg_i$ equal, and thus silent alignment is not happening with Spectral GD. If one was to fully formally analyze the dynamics of Spectral GF without the above assumption, decomposing each $\bk_i, \bq_i$ along all of $\bg_i$ may be crucial. We leave this formality for future work, and show the simulations without this assumption in Figures \ref{fig:tr_dyn}a), \ref{fig:attn_3d} and \ref{fig:attn_20h}.

Under this assumption at initialization, following from the dynamics in Equation \ref{eq:linattn-matrix}, the matrices $\bK, \bQ, \bV$ and $\bL$ stay jointly diagonalizable through training, and the $\texttt{diag}$ operator in Equation \ref{eq:linattn-matrix}a) can be removed. Indeed, if we observe the dynamics at initialization, we have:
\begin{align*}
     &\bC(0):=\bL^2 - \bL \bQ^\top(0) \bV(0) \bK(0) \mathbb{E} \left[ \bLh^2 \right]= \left(\sum_{j=1}^{D} \lambda_j^2 \bg_j \bg_j^\top\right) \\
     &- \left(\sum_{j=1}^{D} \hat{\lambda}_j \bg_j \bg_j^\top\right) \left(\sum_{j=1}^{D} \sigma_j^K(0) \be_j \bg_j^\top\right)^\top \left(\sum_{i=1}^{H} v_i(0) \be_i \be_i^\top\right) \left(\sum_{j=1}^{D} \sigma_j^Q(0) \be_j \bg_j^\top\right) \left(\sum_{j=1}^{D} \lambda_j \bg_j \bg_j^\top\right) \\
    &= \sum_{j=1}^{D} \lambda_j^2 \bg_j \bg_j^\top - \sum_{j=1}^{D} \hat{\lambda}_j \sigma_j^K(0) v_j(0) \sigma_j^Q(0) \lambda_j \bg_j \bg_j^\top \\
    &= \sum_{j=1}^{D} \left( \lambda_j^2 - \lambda_j \hat{\lambda}_j v_j(0) \sigma_j^K(0) \sigma_j^Q(0) \right) \bg_j \bg_j^\top.
\end{align*}

\begin{align*}
    \frac{d\bV(0)}{dt} &= \texttt{diag}\left( \left(\sum_{j=1}^{D} \sigma_j^K(0) \be_j \bg_j^\top\right) \bC(0) \left(\sum_{j=1}^{D} \sigma_j^Q(0) \be_j \bg_j^\top\right)^\top \right) \\
    &= \sum_{j=1}^{H} \sigma_j^K(0) \sigma_j^Q(0) \left( \lambda_j^2 - \lambda_j \hat{\lambda}_j v_j(0) \sigma_j^K(0) \sigma_j^Q(0) \right) \be_j \be_j^\top, \\[1em]
    \frac{d\bK(0)}{dt} &= \left(\sum_{i=1}^{H} v_i(0) \be_i \be_i^\top\right) \left(\sum_{j=1}^{D} \sigma_j^Q(0) \be_j \bg_j^\top\right) \bC(0) \\
    &= \sum_{j=1}^{D} v_j(0) \sigma_j^Q(0) \left( \lambda_j^2 - \lambda_j \hat{\lambda}_j v_j(0) \sigma_j^K(0) \sigma_j^Q(0) \right) \be_j \bg_j^\top, \\[1em]
    \frac{d\bQ(0)}{dt} &= \left(\sum_{i=1}^{H} v_i(0) \be_i \be_i^\top\right) \left(\sum_{j=1}^{D} \sigma_j^K(0) \be_j \bg_j^\top\right) \bC(0) \\
    &= \sum_{j=1}^{D} v_j(0) \sigma_j^K(0) \left( \lambda_j^2 - \lambda_j \hat{\lambda}_j v_j(0) \sigma_j^K(0) \sigma_j^Q(0) \right) \be_j \bg_j^\top.
\end{align*}
We observe that the dynamics completely decouple along the principal components at initialization ($t=0$). Consequently, this decoupled structure is preserved for all time $t \ge 0$, regardless of whether standard or spectral gradient flow is employed. Let the solutions at time $t$ be
\begin{align}  \label{eq:linattn-time_t}
\bV(t) = \sum_{i=1}^{H} v_i(t) \be_i \be_i^\top, \quad 
\bK(t) = \sum_{j=1}^{D} \sigma_j^K(t) \be_j \bg_j^\top, \quad 
\bQ(t) = \sum_{j=1}^{D} \sigma_j^Q(t) \be_j \bg_j^\top.
\end{align}

\subsection{Gradient Flow.}
\label{appendix: gf lin attn theorem}

\begin{theorem}[Gradient flow on linear attention]
    Under Assumption \ref{ass:linattn}, when optimized by gradient flow, the total weights of linear attention evolve as
\begin{align*}
\bQ^\top \bV \bK = \sum_{j=1}^D v_j(t)^3 \bg_j \bg_j^\top ,
\end{align*}
where the magnitude in each eigen-direction $v_j \, (j=1,\cdots,D)$ is given by
\begin{align*}
\frac{d v_j}{dt} &= v_j^2 \left( \lambda_j^2 - \lambda_j \hat\lambda_j v_j^3 \right). \\
\end{align*}
This gives rise to time separation between heads being learned, where head $j$ is learned in time $\propto 1/(\lambda_j^2 \epsilon_j).$
\end{theorem}

\begin{proof}
Substituting \cref{eq:linattn-time_t} into \cref{eq:linattn-matrix}, we obtain
\begin{align*}
\frac{d}{dt}\sum_{i=1}^{H} v_i(t) \be_i \be_i^\top
&= \sum_{j=1}^D \sigma_j^K (t)\sigma_j^Q(t) \left( \lambda_j^2 - \lambda_j \hat\lambda_j v_j(t) \sigma_j^K(t) \sigma_j^Q(t) \right) \be_j \be_j^\top ,\\
\frac{d}{dt}\sum_{j=1}^{D} \sigma_j^K(t) \be_j \bg_j^\top
&= \sum_{j=1}^D v_j(t) \sigma_j^Q(t) \left( \lambda_j^2 - \lambda_j \hat\lambda_j v_j(t) \sigma_j^K(t) \sigma_j^Q(t) \right) \be_j \bg_j^\top ,\\
\frac{d}{dt}\sum_{j=1}^{D} \sigma_j^Q(t) \be_j \bg_j^\top
&= \sum_{j=1}^D v_j(t) \sigma_j^K(t) \left( \lambda_j^2 - \lambda_j \hat\lambda_j v_j(t) \sigma_j^K(t) \sigma_j^Q(t) \right) \be_j \bg_j^\top .
\end{align*}
Equating both sides, we get for $j=1,\cdots,D$
\begin{subequations}
\begin{align*}
\frac{d v_j}{dt} &= \sigma_j^K \sigma_j^Q \left( \lambda_j^2 - \lambda_j \hat\lambda_j v_j \sigma_j^K \sigma_j^Q \right) ,\\
\frac{d \sigma_j^K}{dt} &= v_j \sigma_j^Q \left( \lambda_j^2 - \lambda_j \hat\lambda_j v_j \sigma_j^K \sigma_j^Q \right) ,\\
\frac{d \sigma_j^Q}{dt} &= v_j \sigma_j^K \left( \lambda_j^2 - \lambda_j \hat\lambda_j v_j \sigma_j^K \sigma_j^Q \right) .
\end{align*}
\end{subequations}
Because $v_j,\sigma_j^K,\sigma_j^Q$ are initialized to be equal and will stay equal throughout training, the dynamics for each eigen-direction can be summarized by one separable, ordinary differential equation
\begin{align}  \label{eq:linattn-ode}
\frac{d v_j}{dt} &= v_j^2 \left( \lambda_j^2 - \lambda_j \hat\lambda_j v_j^3 \right) ,\quad j=1,\cdots,D .
\end{align}
\cref{eq:linattn-ode} does not admit a closed-form solution of $v_j(t)$ in terms of $t$. However, its behavior is easy to predict from \cref{eq:linattn-ode}. With small initialization $v_j(0)=\epsilon_j$, $v_j(t)$ undergoes a plateau of duration $T_j \approx 1/(\lambda_j^2 \epsilon_j)$. This is because the initial dynamics behaves like $\dot v_j = v_j^2\lambda_j^2$, until $v_j^2\lambda_j^2=\Theta(1)$. After the plateau, $v_j(t)$ grows rapidly and converges to $v_j(\infty)=\left( \frac{\lambda_j}{\hat\lambda_j} \right)^{1/3}$.

The overall weight thus evolves as
\begin{align*}
\bQ^\top \bV \bK = \sum_{j=1}^D v_j(t)^3 \bg_j \bg_j^\top ,
\end{align*}
where $v_j(t)$ solves the differential \cref{eq:linattn-ode}. Since each eigen-direction grows after a plateau of duration $T_j \approx 1/(\lambda_j^2 \epsilon_j)$, distinct eigenvalues $\lambda_j$ or distinct initial conditions $\epsilon_j$ in each head will lead to stage-like dynamics, in which different eigen-directions are learned at different times.
\end{proof}

\subsection{Spectral Gradient Flow.}
\label{appendix: spectral gd lin attn}

\begin{theorem}[Spectral gradient flow on linear attention]
        Under Assumption \ref{ass:linattn}, when optimized by Spectral gradient flow, the total weights of linear attention evolve as
\begin{align*}
\bQ^\top \bV \bK = \sum_{j=1}^D \min\left[ t^3, \frac{\lambda_j}{\hat \lambda_j}\right] \bg_j \bg_j^\top ,
\end{align*}
and the magnitude in each eigen-direction $v_j \, (j=1,\cdots,D)$ is given by
\begin{align*}
v_j(t) &= \min\bigg[t, \left(\frac {\lambda_j}{\hat \lambda_j}\right)^{1/3}\bigg]
\end{align*}
Contrary to time separation in the GF case, Spectral GF learns all heads at the same time, with the same rate, where head $j$ is learned in time $\propto \left(\frac {\lambda_j}{\hat \lambda_j}\right)^{1/3}.$
\end{theorem}

\begin{proof} 
Similar to the gradient flow case, we observe the dynamics in SVD decomposed form, but with orthogonalization, denoted by \texttt{orth}:
\begin{align*}
\frac{d}{dt}\sum_{i=1}^{H} v_i(t) \be_i \be_i^\top
&=\texttt{orth}\left( \sum_{j=1}^D \sigma_j^K (t)\sigma_j^Q(t) \left( \lambda_j^2 - \lambda_j \hat\lambda_j v_j(t) \sigma_j^K(t) \sigma_j^Q(t) \right) \be_j \be_j^\top\right) ,\\
\frac{d}{dt}\sum_{j=1}^{D} \sigma_j^K(t) \be_j \bg_j^\top
&= \texttt{orth}\left(\sum_{j=1}^D v_j(t) \sigma_j^Q(t) \left( \lambda_j^2 - \lambda_j \hat\lambda_j v_j(t) \sigma_j^K(t) \sigma_j^Q(t) \right) \be_j \bg_j^\top\right) ,\\
\frac{d}{dt}\sum_{j=1}^{D} \sigma_j^Q(t) \be_j \bg_j^\top
&= \texttt{orth}\left(\sum_{j=1}^D v_j(t) \sigma_j^K(t) \left( \lambda_j^2 - \lambda_j \hat\lambda_j v_j(t) \sigma_j^K(t) \sigma_j^Q(t) \right) \be_j \bg_j^\top\right) .
\end{align*}
Dynamics again decouples, and for each head $j=1,\cdots,D$ we have:
\begin{subequations}
\begin{align*}
\frac{d v_j}{dt} &= \mathds 1_{\{v_j \sigma_j^K \sigma_j^Q<\lambda_j/\hat\lambda_j\}} ,\\
\frac{d \sigma_j^K}{dt} &= \mathds 1_{\{v_j \sigma_j^K \sigma_j^Q<\lambda_j/\hat\lambda_j\}} ,\\
\frac{d \sigma_j^Q}{dt} &= \mathds 1_{\{v_j \sigma_j^K \sigma_j^Q<\lambda_j/\hat\lambda_j\}} .
\end{align*}
\end{subequations}
Thus, using that $v_j(0)=\sigma_j^K(0)=\sigma_j^Q(0)=\epsilon_j$, their evolution is characterized by
\begin{align*}
\frac{d v_j}{dt} &= \mathds 1_{\{v_j(t)^3<\lambda_j/\hat\lambda_j\}} ,\quad j=1,\cdots,D .
\end{align*}
This exactly results in 
$$v_j(t) = \min\bigg[t, \left(\frac {\lambda_j}{\hat \lambda_j}\right)^{1/3}\bigg],$$
hence the head $j$ is learned in time $T_j\propto \left(\frac {\lambda_j}{\hat \lambda_j}\right)^{1/3}$, and the full weights dynamics are
\begin{align*}
\bQ^\top \bV \bK = \sum_{j=1}^D \min\left[ t^3, \frac{\lambda_j}{\hat \lambda_j}\right] \bg_j \bg_j^\top.
\end{align*}
When $N\to\infty$, $\hat \lambda_j\to\lambda_j^2$ as can be seen from Equation \ref{eq:lambdas}, and thus the time for learning the head $j$ in this regime is $T_j\propto \lambda^{-1/3}$.

In the case of Spectral GF, all eigen-directions are being learned at the same time and with the same speed, and no stage-like dynamics is present. Notice also that, while the convergence time in the case of GF depends on the initial conditions $\epsilon_j$, the dependence is no longer present in Spectral GF.
\end{proof}

\section{Further Experiments and  Details}
\label{appendix: details}
We note that all the experiments in the paper are small scale, they were ran on a single L40 GPU, requiring not more than 8GB of RAM. The experiments take minutes to finish individually.
\subsection{Deep Linear Networks}
\label{appendix: dln}
More analogous experiments to the ones in Figure \ref{fig:theory} can be found in Figure \ref{fig:more_theory}. 

Additionally, we show the oscillations of the singular values of $\bW$ during training in Figures \ref{fig:osc},\ref{fig:oscmom}, happening as a consequence of our discussion in Appendix \ref{appendix: theory}. 
\begin{figure}
    \centering
    \includegraphics[width=0.9\linewidth]{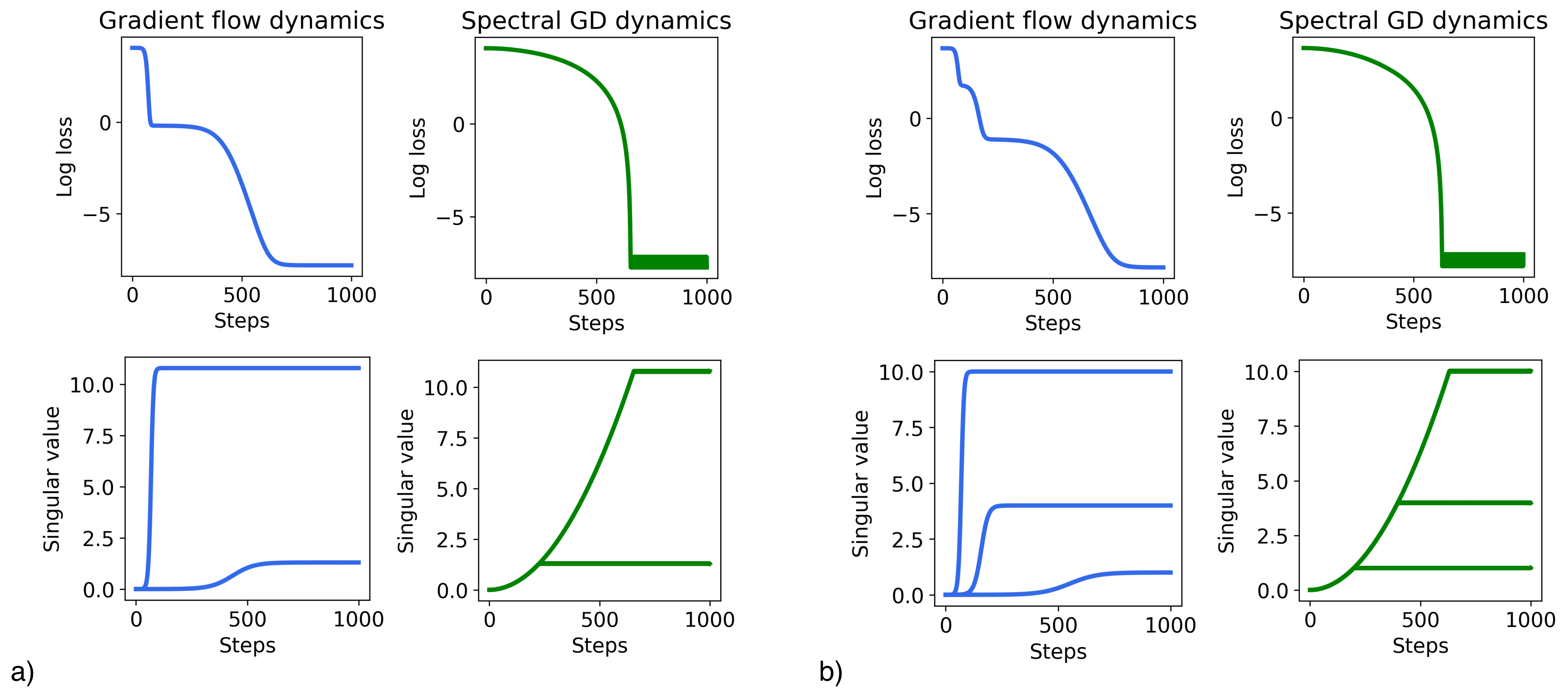}
    \caption{Additional figures supporting the theory from Section \ref{sec: dlnn}. a) $d_\text{in} = d_\text{out}=2$; b) $d_\text{in} = d_\text{out}=3$}
    \label{fig:more_theory}
\end{figure}
\begin{figure}
    \centering
    \includegraphics[width=\linewidth]{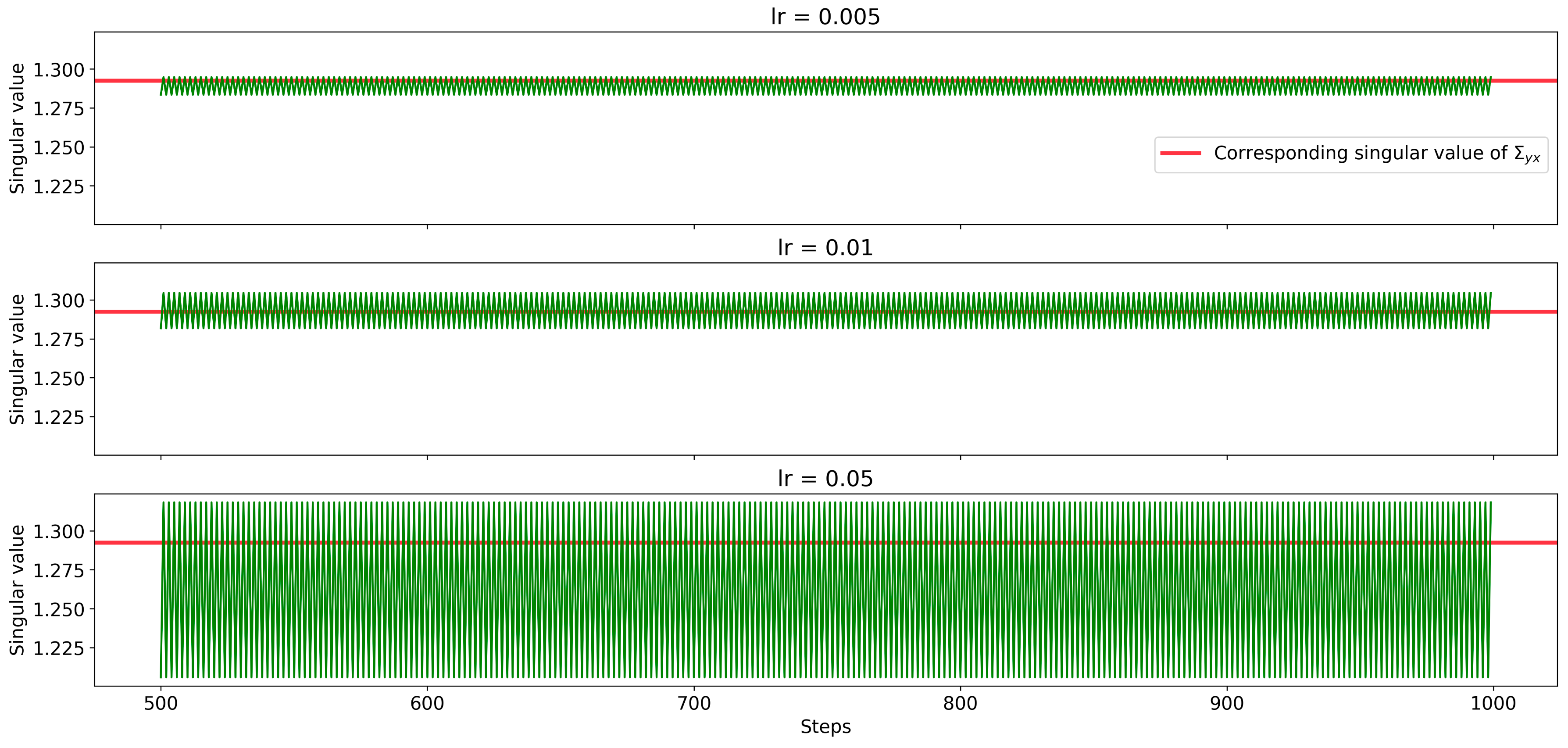}
    \caption{The oscillations of $\sigma_2(t)$ around $s_2$ in the setting from Figure \ref{fig:theory}, shown for different values of the learning rate. This phenomena happens after a principal component is effectively learned by Spectral GD, but not exactly. Then the small noise in the direction of that principal component is amplified by orthogonalization, and the step of order 1 is taken, independently of noise magnitude.}
    \label{fig:osc}
\end{figure}
\begin{figure}
    \centering
    \includegraphics[width=\linewidth]{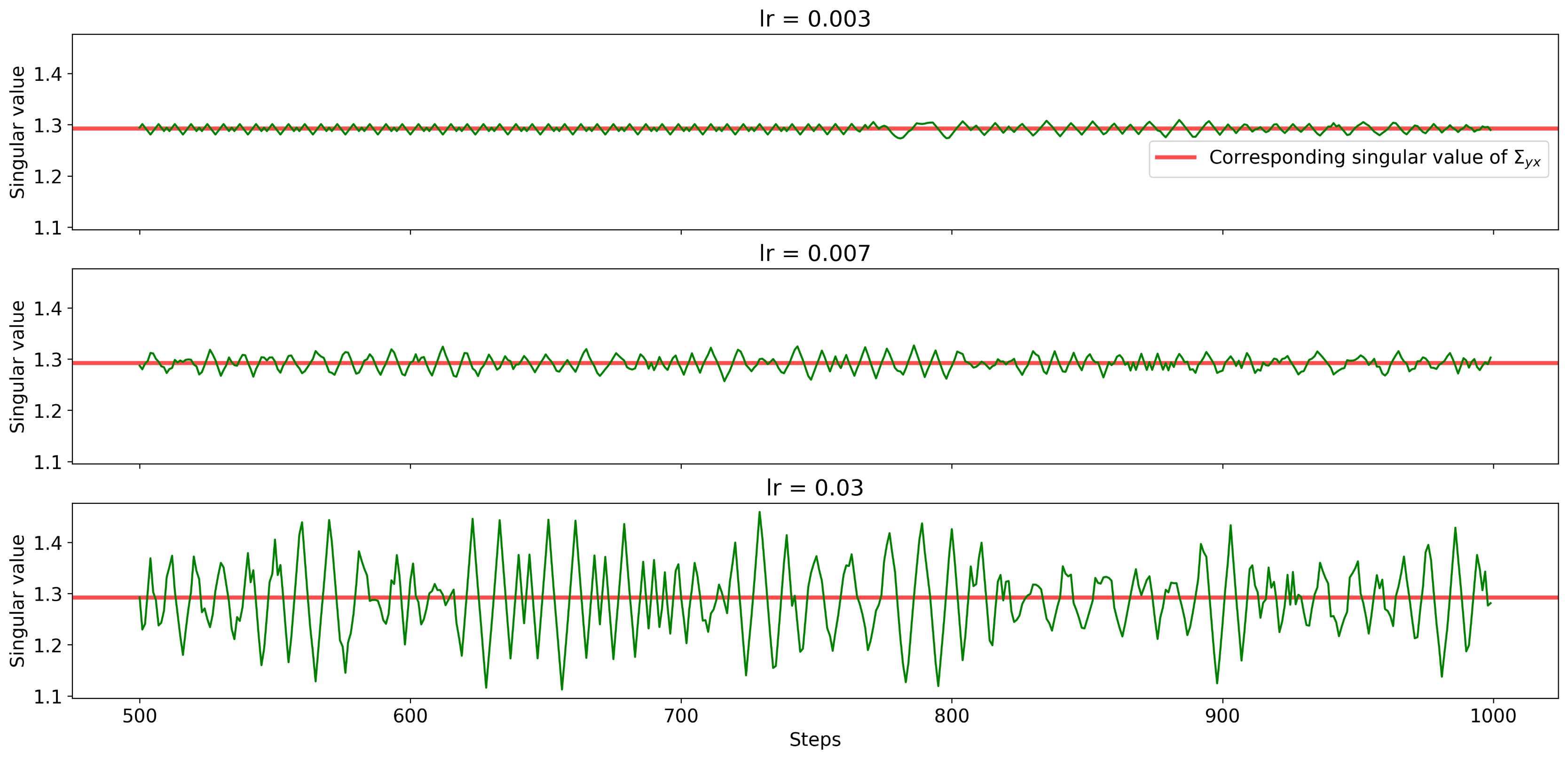}
    \caption{Same setup as Figure \ref{fig:osc}, but this time we're using Spectral GD with momentum, a step closer to recovering Muon optimization.}
    \label{fig:oscmom}
\end{figure}

\subsection{Digit and Pixel Feature Learning}
\label{appendix:features}
Here we provide more implementation details on the results in Figures \ref{fig:features} and \ref{fig:dp_adam}. For all three optimizers we picked the best performing learning rates across multiple runs. The learning rates sweep results are shown in Figure \ref{fig:lr_sweep}. We used Muon implementation \texttt{SingleDeviceMuonWithAuxAdam} from \citet{jordan2024muon}, with the base learning rate applied to Adam optimized parameters, and for the Muon ones the learning rate was $80\times$ larger. Muon optimized parameters had weight decay set to $0.1$, and momentum for both SGD and Muon was 0.95. Adam's beta parameters were set to $(0.9, 0.999)$. For training batch size was set to 64, and we were evaluating results with checkpoints every 100 batches. In total, the base run had 10 epochs. Default pixel feature has intensity 1. We used a simple CNN with  2 convolutional layers followed by 2 linear layers. Weights are initialized as Gaussian, with mean 0 and standard deviation 0.01. Convolution kernel sizes are 3, and the number of filters is 32 in the first and 64 in the second layer, with MaxPool of size 2 in between. The first fully connected layer has output dimension 128, and the second one 10 (the number of classes).
\begin{figure}
    \centering
    \includegraphics[width=\linewidth]{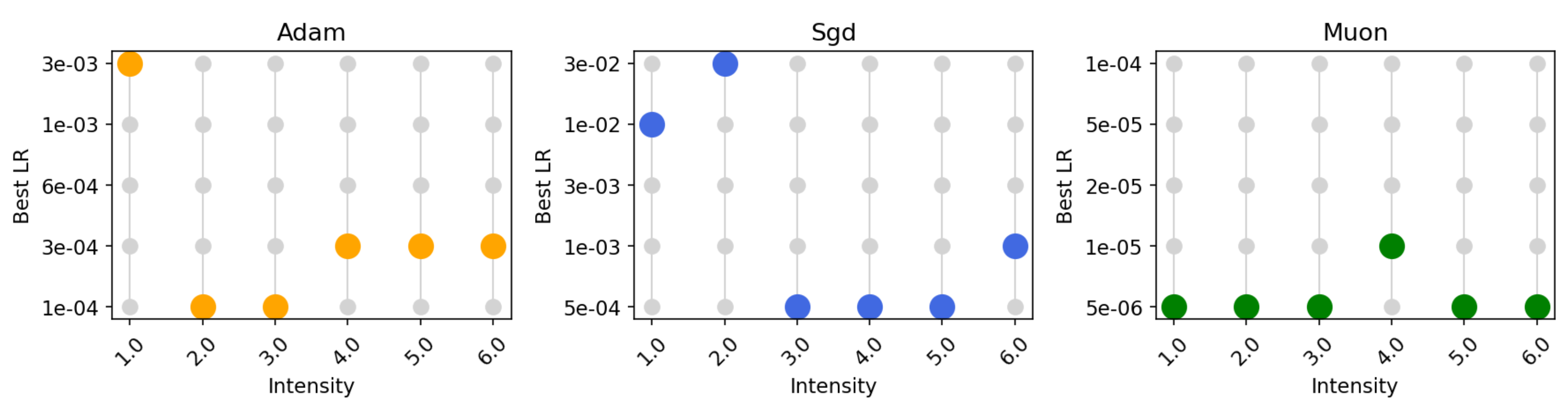}
    \caption{Learning rates used for experiments with pixel and digit feature in Section \ref{sec: non-lin}.}
    \label{fig:lr_sweep}
\end{figure}

\subsection{Routing Task and Shared Representations}
\label{appendix: nrr}
\paragraph{Setup Details.} Here we clarify the setup used for experiments in Section \ref{sec: routing}. We denote a sample $(x_i, y_i)_{j, o}$ to mean that $x_i$ comes from input source $j$ and $y_i$ is going to the output source $o$. During training, we only see samples $(x_i, y_i)_{j, j}$ and $(x_i, y_i)_{j, j+1}$. Denote with $W^{in}_{j}\in \Rl^{64\times 4}$ the input encoder for source $j$, $W^{out}_o\in \Rl^{7\times 64}$ the decoder for output source $o$, and $W^{h}\in \Rl^{64\times 64}$ the hidden linear layer. Then given $(x_i, y_i)_{j, o}$, the output of the network is $\hat y_i = W^{out}_o W^h W^{in}_jx_i$. We use MSE objective to train the network, i.e. minimize $\|\hat y_i-y_i\|^2$. Each gradient update involves a batch containing a single sample from each allowed input-output pair. Full batch sampling procedure can be seen in Algorithm \ref{alg:batch_sampling}.

\begin{algorithm}[ht]
\caption{Sampling a batch in the routing task}
\label{alg:batch_sampling}
\begin{algorithmic}[1]
\REQUIRE Number of sources $m=7$, allowed shifts $k=2$, total numbers in the task $N=4$, orthonormal set of $N$ input vectors in $\R^N$ for each input source $j$: $\{(v_1^j, ..., v_N^j)\}_{j=0,...,m-1}$
\STATE Initialize batch $\mathcal{B} \gets \emptyset$
\FOR{input source $j = 0$ \TO $m - 1$}
    \FOR{shift $s = 0$ \TO $k - 1$}
        \STATE Output source $o \gets (j + s) \bmod m$
        \STATE Sample a random data index $i \sim \mathcal{U}(1, N)$
        \STATE $x_i\gets v_i^j$
        \STATE $y_i\gets$ output vector corresponding to $i$ (Figure \ref{fig:nrr}b)
        \STATE Add sample to batch: $\mathcal{B} \gets \mathcal{B} \cup \{(x_i, y_i)_{j, o}\}$
    \ENDFOR
\ENDFOR
\RETURN Training batch $\mathcal{B}$
\end{algorithmic}
\end{algorithm}

\paragraph{Analysis.}In a similar setup, \citet{saxe2022neural} theoretically demonstrate that GD has an implicit bias towards learning shared representations, with the core argument lying in Theorem \ref{thm:gf} 2. Because of our gained intuition, we identified this setting as the one where Spectral GD may be disadvantaged by the absence of simplicity bias. Mathematically, each gating mechanism follows the dynamics derived by \citet{saxe2022neural}:
\begin{align*}
    \frac{d}{dt} B_1 &= \frac{\sqrt{P}}{M^2} B_2 B_1 \left[ S - B_2 B_1^2 D \right] \\
    \frac{d}{dt} B_2 &= \frac{P}{M^2} B_1^2 \left[ S - B_2 B_1^2 D \right]
\end{align*}
Here, $B_1$ and $B_2$ represent the spectra/mode strengths of the gated components, while $S$ and $D$ denote certain data statistics (see \citet{saxe2022neural}). The focus is on ``pathway counting'' argument: the gradient updates are scaled by pre-factors proportional to $P$, the number of pathways passing through a gated unit. In the initial phase of learning, where dynamics are dominated by exponential growth, these pre-factors act as rate constants; consequently, configurations with maximal $P$ (i.e., $P=M^2$, corresponding to shared representations) are exponentially faster to learn and win the ``neural race''. In contrast, Spectral GD orthogonalizes each of the updates, effectively removing the dependence of learning speed on the pathway multiplicity $P$, and the order of speed of learning for each gating mechanism is the same. In that case, the winning gating strategy is dependent on other conditions, such as the random initialization.

\subsection{Transformer in Cycle Task}
\label{appendix: transformer}
Implementation details related to the transformer model are as follows. We use two layer softmax attention transformer, with 2 heads in each layer, with causal masking and RoPE positional embedding. The token dimension being equal to the embedding dimension was $N=8$, and the inputs were one-hot encodings of each number. The rank of each head was 12. The transformer was attention-only, meaning it didn't have MLPs. The residual connections were present after the first layer, but not after the second (as it was the last). Every weight matrix of dimensions $(\texttt{in\_dim, out\_dim})$ was initialized as Gaussian with mean $\mathbf 0$ and std $0.01/\sqrt{\texttt{in\_dim}\times \texttt{out\_dim}}$. For SGD, learning rate was set to 0.2 for the second layer and 5 times larger for the first (as otherwise the model didn't learn) and momentum to 0.9. As no weights used biases, for Muon run all weights were optimized with \texttt{SingleDeviceMuon} from \citet{jordan2024muon}. The learning rate was 0.0003 with the default momentum of 0.95. Training was done by minimizing MSE loss for each next token prediction. We train for 40,000 steps, each one involving a single batch of size 512. We used $N=8$ numbers in a cycle and sequence length $80$.
\paragraph{Further experiments.} We reproduce the phenomenon across couple of different setups. In Figure \ref{fig:tr_diff_set}a), we show the results when training on the skips $\{1,2,3,-2,-3\}$ and observe whether the transformers generalize to skip -1. In b) we train with transformers with 3 heads, each rank 12, and in c) with 4 heads. 
\begin{figure}
    \centering
    \includegraphics[width=\linewidth]{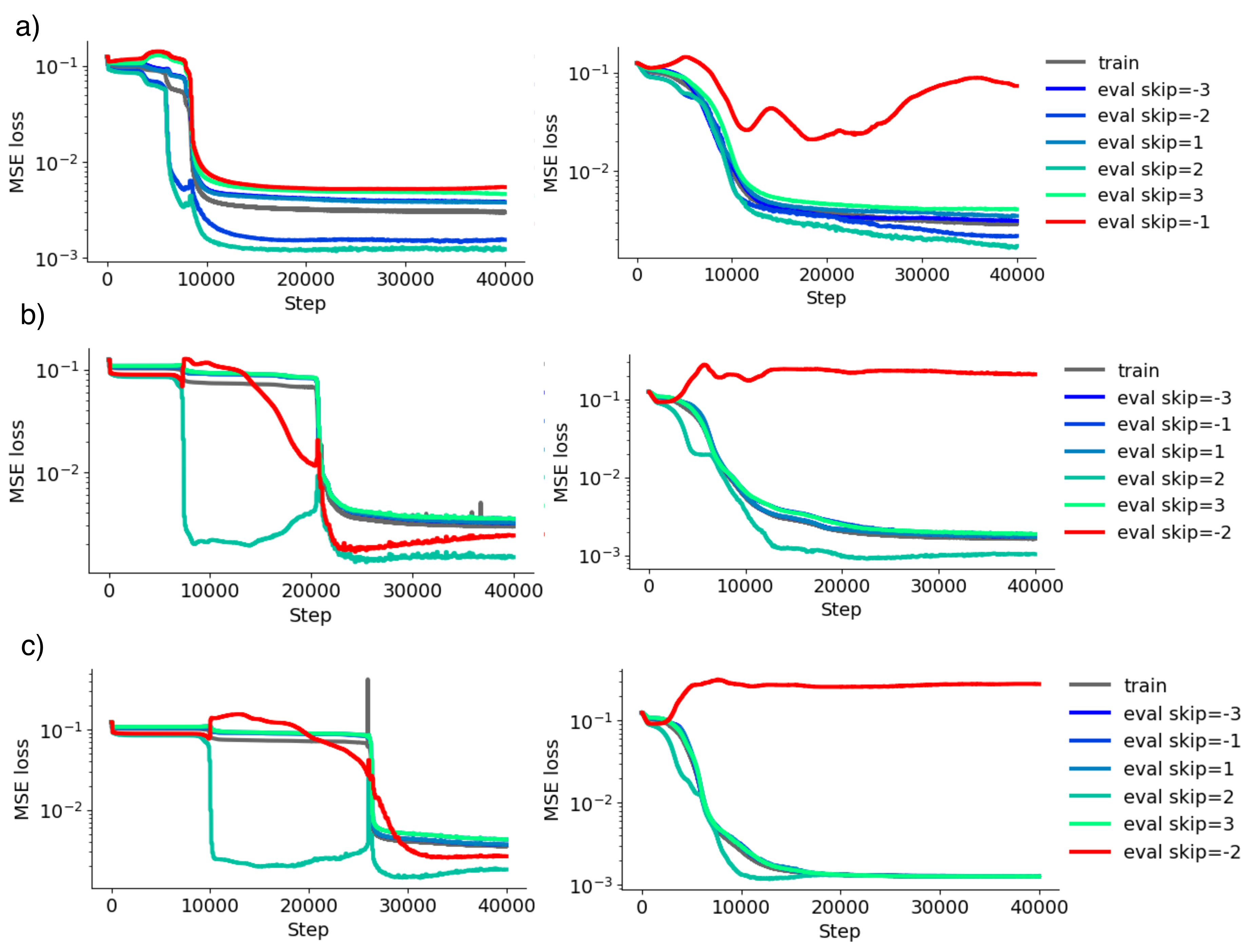}
    \caption{a) SGD (Left) and Muon (Right) optimized transformers, trained on skips $\{1,2,3,-2,-3\}$, and evaluated at skip -1. b) The same run as in the base setting, but with 3 heads per layer instead. c) 4 heads instead. All plots show the same phenomena observed in Section \ref{sec: transformer}: SGD (Left) generalizes to the unseen skip, and Muon (Right) doesn't. }
    \label{fig:tr_diff_set}
\end{figure}
\subsection{Linear Attention Simulations Beyond the Assumptions}
\label{appendix: lin attn sim}
For the simulations in Figure \ref{fig:tr_dyn}, we used the task setup described in Section \ref{sec: transformer_theory}. While the figure is illustrating the learning dynamics of Spectral GD, for its generation we used Muon. The linear attention model consists of three matrices $\bK, \bQ, \bV$ that were optimized either by GD with learning rate $0.01$ or Muon with learning rate $0.00023$. The initialization of $\bK, \bQ$ is Gaussian with mean 0 and std $\texttt{init}/\sqrt{HD}$, where the $\texttt{init}=0.01$ for GD and $\texttt{init}=0.001$ for Muon. $\bV$ is initialized as diagonal, with diagonal entries being Gaussian with mean 0 and std $\texttt{init}/H$. We do not further constrain it to stay diagonal through training, although empirically it seems to be the case. Note that this setting doesn't impose the joint diagonalizability assumption directly, but only through small initialization. For GD this is enough as the growth happens along one direction at a time, aligning the key and query vectors---a phenomenon called silent alignment. For Muon this is not the case, as the key and query vectors grow equally in all directions of eigenvectors of $\bL$. Nonetheless, the general phenomenology of all heads being learned in the same time still holds. 

We show more simulations for $D=H=3$ in Figure \ref{fig:attn_3d}, and for $D=2, H=20$ in Figure \ref{fig:attn_20h}. For the simulations applying Spectral GD to full matrices $\bW_K, \bW_Q, \bW_V$ instead of the reparameterized $\bK, \bQ,\bV$, see Figure \ref{fig:attn_full_matrices}, showing the learning of all heads simultaneously still happens.
\begin{figure}
    \centering
    \includegraphics[width=\linewidth]{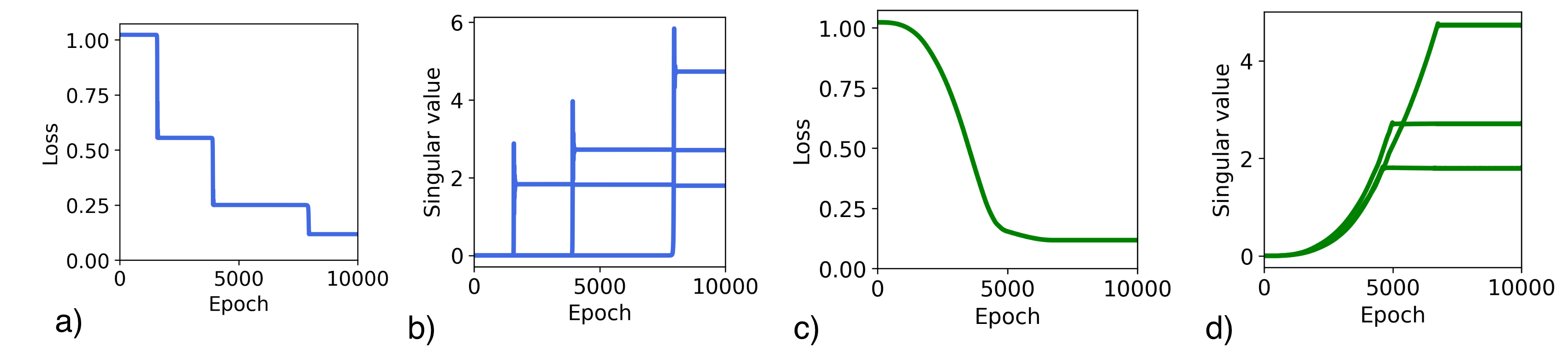}
    \caption{a) Loss for GD, b) evolution of singular values of $\bK\bV\bQ^\top$ for GD optimized attention, c) loss for Spectral GD d) singular values for Spectral GD in the setting $D=H=3, R=1$.} 
    \label{fig:attn_3d}
\end{figure}
\begin{figure}
    \centering
    \includegraphics[width=\linewidth]{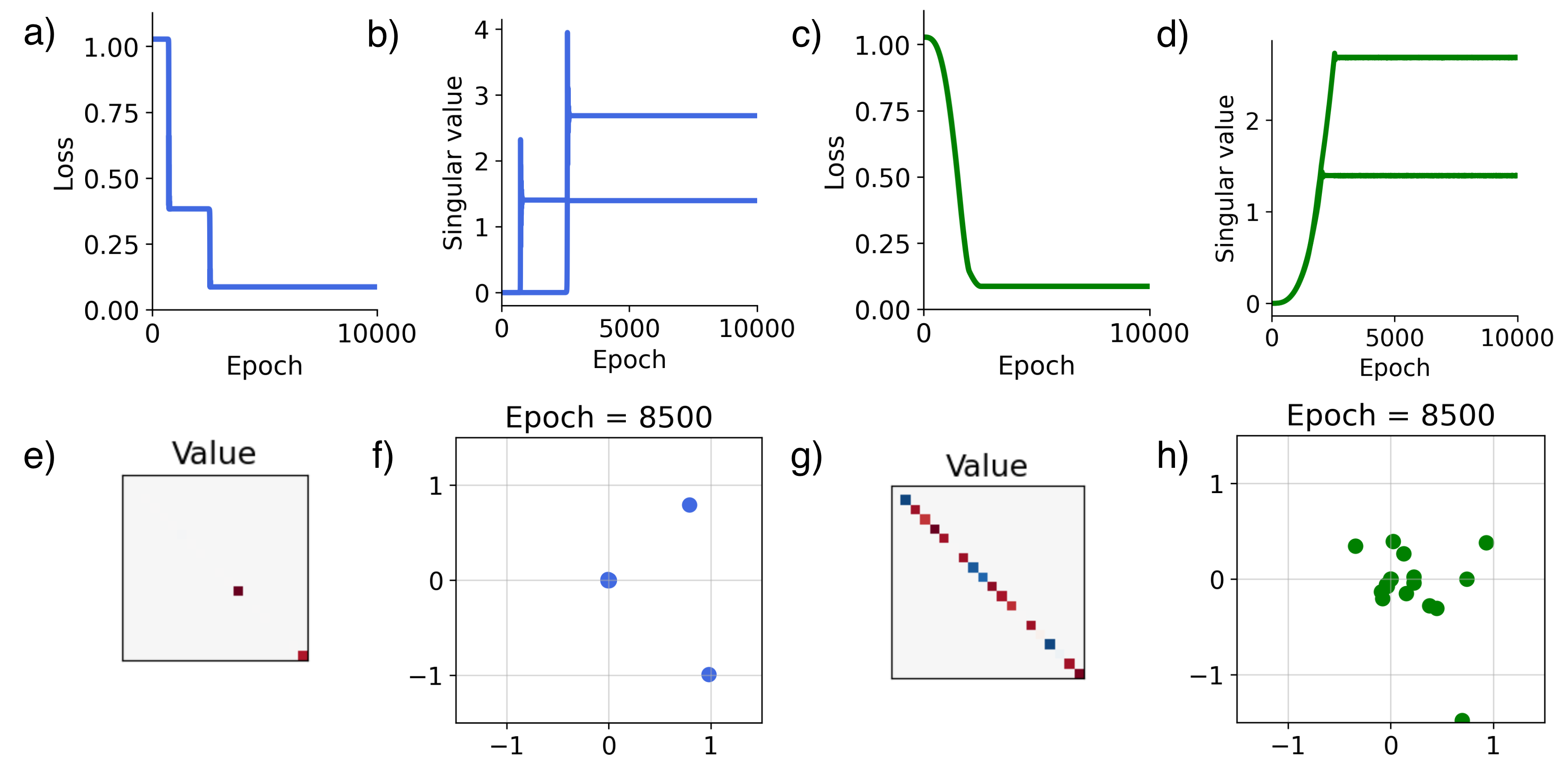}
    \caption{Linear attention learning when $H>D$, here $H=20, D=2$. a) GD loss; b) GD singular values; c) Spectral GD loss; d) Spectral GD singular values; e) GD final value matrix; f) GD key vectors at step 8500; g) Spectral GD final value matrix; h) Spectral GD key vectors at step 8500. All together, we notice how GD activates as many heads as the task requires, but not more. On the other hand, Spectral GD activates all the heads, showcasing its loss of simplicity bias.}
    \label{fig:attn_20h}
\end{figure}
\begin{figure}
    \centering
    \includegraphics[width=\linewidth]{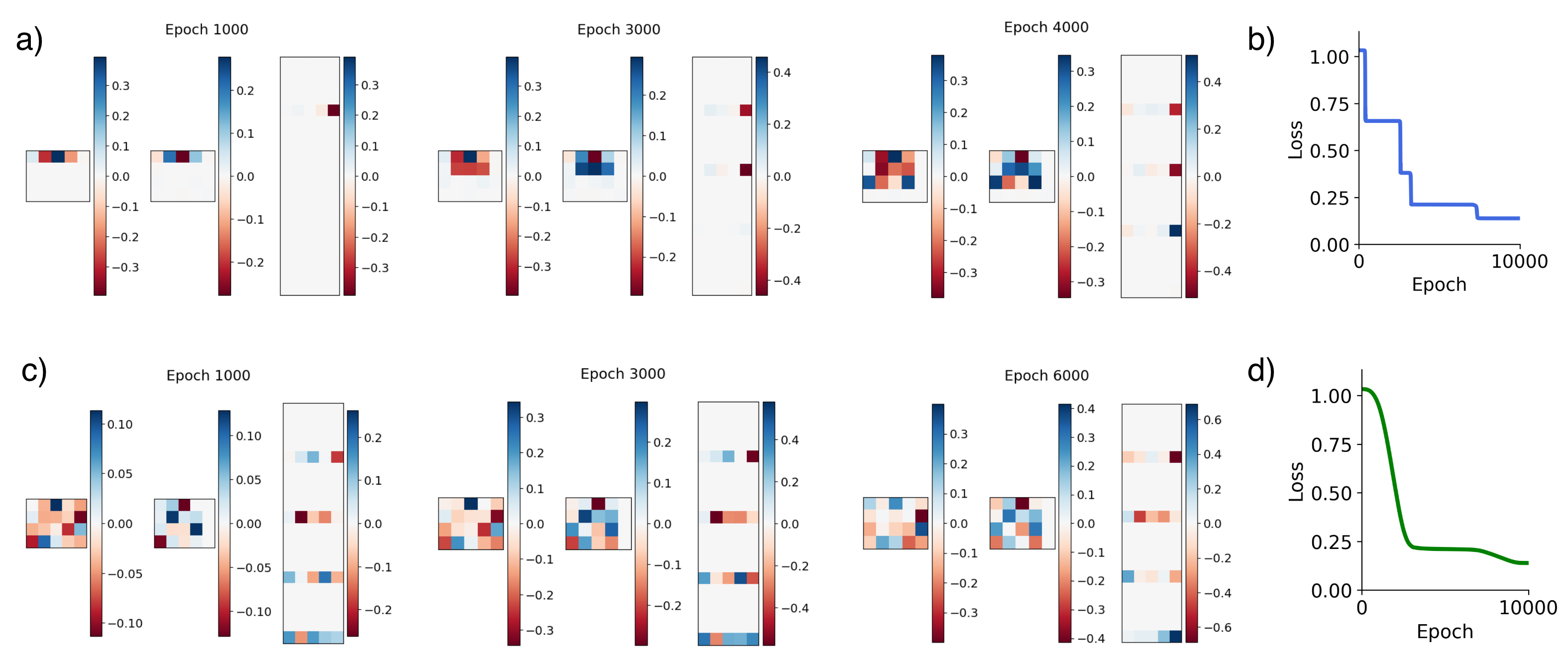}
    \caption{Empirical observations of weight matrices (a,c) $\bW_K, \bW_Q, \bW_V$ in that order, through training and loss curves (b,d) when trained with GD (a,b) and Spectral GD (c,d).}
    \label{fig:attn_full_matrices}
\end{figure}

\section{Questions and Answers}
\label{appendix: qa}
Many sub-questions arose while we were writing the paper. For the flow of presentation, we didn't include them in the main text, but instead collect them here.
\begin{itemize}
    \item \textbf{What about Adam optimizer?} For the theory part, Adam's per-parameter update rule makes it hard to mathematically derive the evolution of singular values of its solution matrix. There is no guarantee even that the principal components will be aligned through training, even if they are at initialization. Nonetheless, it is the case that Adam also doesn't follow the same trajectory as GD, escaping the saddles \citep{jacobs2026never}. In fact, an Adam simplification frequently used by theorists, named SignGD, employed on a simplified, diagonal deep linear network has the exact same dynamics (mathematically) as Spectral GD in the same setting. All that to say, Adam also loses the simplicity bias of GD. Below we show how it compares empirically in each setting we considered.
\end{itemize}
\subsection{Digit and Pixel Feature Learning}
\label{appendix:qa_dp}
\begin{itemize}
    \item \textbf{What if the state where the digit is a dominant feature for SGD is meta-stable, and eventually the pixel becomes dominant?} With all the saddle point plateaus SGD has to escape, we were wondering if the accuracy on Misaligned; label=digit set will eventually drop. We show a longer experiment, this time with 100 epochs in Figure \ref{fig:dp_long}. Noticing how SGD accuracy on Misaligned datasets stabilized, we doubt that the situation will change after further training. 
    \item \textbf{How do we know it is a competition between fully present features rather than one of the features being transient?} To answer this, we fit a linear probe on the last hidden layer. Results in Figure \ref{fig:probe} suggest that throughout training, both features are present, and no transience occurs. This explains that the reason the accuracy on Misaligned; label=digit decreases is the increasing strength of pixel feature, not the decreasing strength of the digit one.
    \item \textbf{Have you plotted average probability of the correct token instead of accuracy?} Probability for each relevant label has a nice property of being continuous per sample, where accuracy on a sample contributes either with 0 or 1. For this reason, we also plot the probabilities on Misaligned datasets in Figure \ref{fig:prob_spur}a,b), analogous to the plots in Figure \ref{fig:features}d,e). The qualitative similarity with the hypothesized feature presence remains.
    \item \textbf{How does Adam perform?} We saw in Section \ref{sec: spur} that in the presence of spurious features Adam can also be worse than SGD, contrary to the story in \citet{vasudevarich}. In Figure \ref{fig:dp_adam} we show the accuracy curves of Adam in the two feature learning setup. 
    \item \textbf{Have you tried the setting where spurious pixel is highly, but not fully correlated with the label?} To this end, we train on a set where $\alpha$ of all images contain the spurious pixel, for $\alpha=0.95$ and $\alpha=0.9$. Peak accuracy for all three optimizers on Digit only set is almost perfect. For that reason, we plot the peak accuracy on Misaligned; label=digit set instead. The results as we vary the intensity of the spurious pixel are shown in Figure \ref{fig:prob_spur}c, d). Once again, SGD is better for low pixel intensities, but as the intensity of spurious pixel grows, the pixel feature becomes more dominant and hurts the performance of SGD more than Muon and Adam.
\end{itemize}
\begin{figure}
    \centering
    \includegraphics[width=0.8\linewidth]{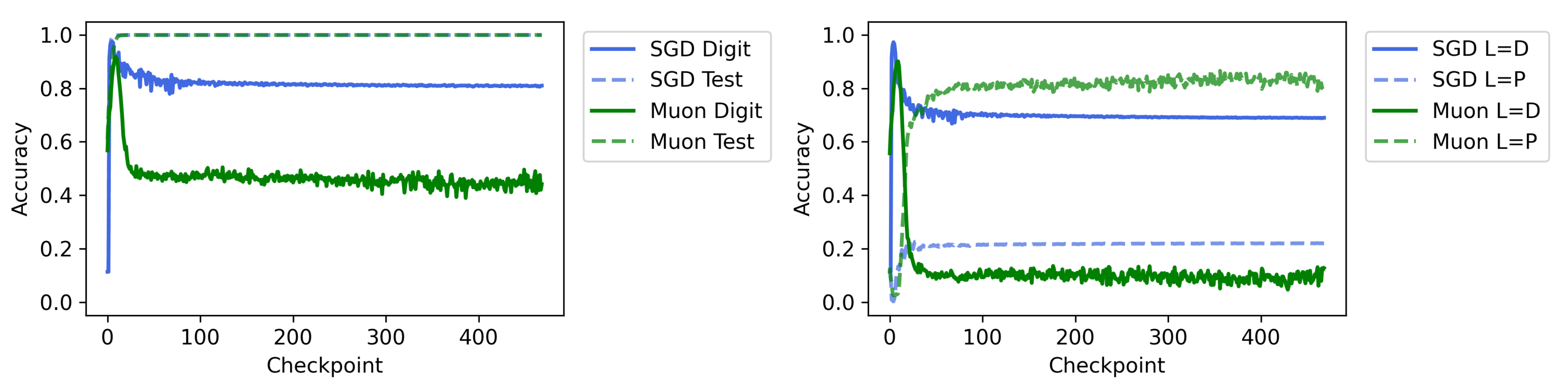}
    \caption{Accuracies of Muon and SGD algorithms ran for 100 epochs on (Left) Test and Digit only and (Right) Misaligned evaluators.}
    \label{fig:dp_long}
\end{figure}
\begin{figure}
    \centering
    \includegraphics[width=0.8\linewidth]{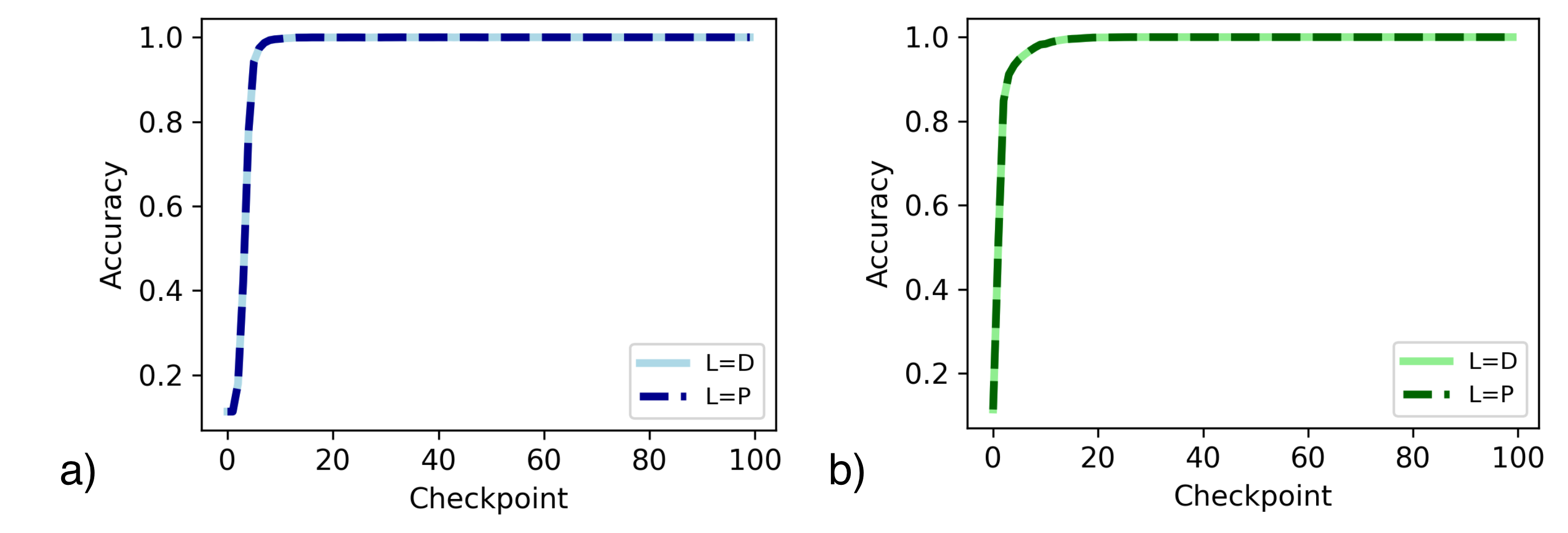}
    \caption{Fitting a linear probe on the final layer of the models trained with a) SGD and b)Muon. The perfect accuracy on both Misaligned sets tells us that both the digit and pixel feature are present.}
    \label{fig:probe}
\end{figure}
\begin{figure}
    \centering
    \includegraphics[width=\linewidth]{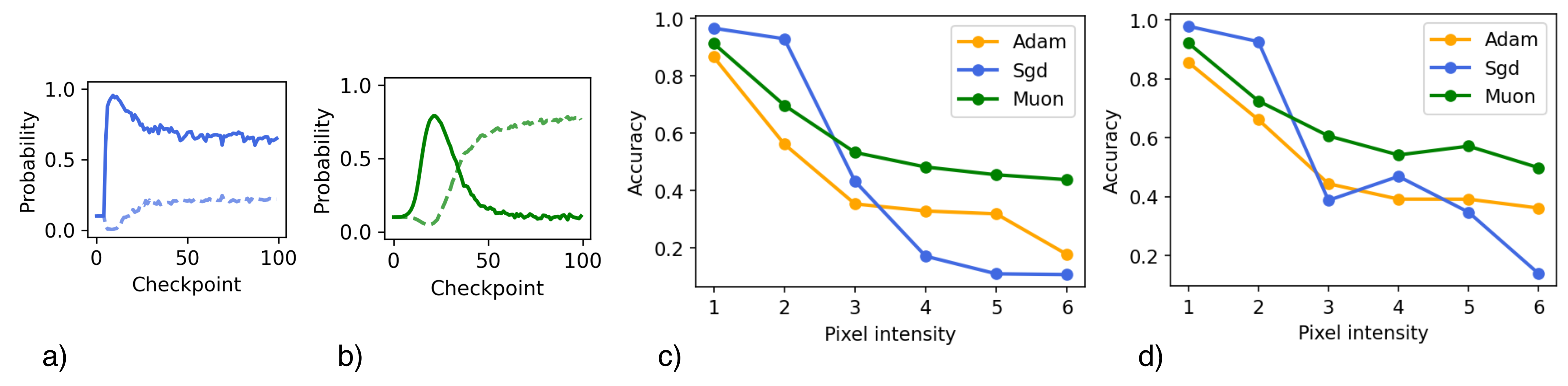}
    \caption{Probabilities of the label=digit (full line) and label=pixel outputs, on the base runs optimized by a) SGD and b) Muon. The curves behave similarly to the accuracy ones in Figure \ref{fig:features}d,e). The peak accuracies on Misaligned; label=digit set when only a) $\alpha=0.95$, b) $\alpha=0.9$ fraction of training samples has the spurious pixel.}
    \label{fig:prob_spur}
\end{figure}
\begin{figure}
    \centering
    \includegraphics[width=0.8\linewidth]{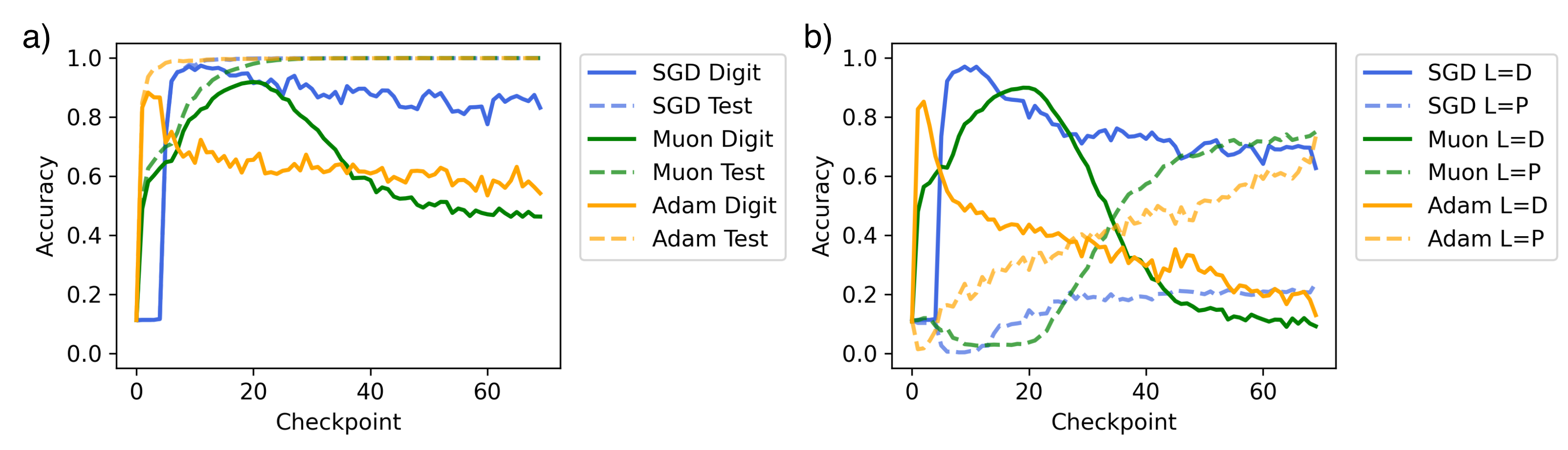}
    \caption{Same setup as in Figure \ref{fig:features}b,d), but now together with Adam's performance.}
    \label{fig:dp_adam}
\end{figure}
\subsection{Transformer Learning Cycles}
\label{appendix: qa_tr}
\begin{itemize}
    \item \textbf{How do Adam and AdamW compare in your setting?} Adam and AdamW, while behaving differently than Muon, also seemingly fail to capture the underlying similarity between different cycle skips. In Figure \ref{fig:tr_adam}, we provide their loss curves.
    \item \textbf{What happens if you switch to cross-entropy loss?} If we use cross-entropy loss instead of MSE, both optimizers seemingly don't learn shared representations, with addition of SGD's loss also exploding eventually (Figure \ref{fig: tr_ce}).
    \item \textbf{What happens if you add Layer Norm?} We train the transformers with a fixed (unlearnable, \texttt{elementwise\_affine=False}) layer norm; for SGD we lowered the learning rate to 0.02. We see that both optimizers seemingly failed to capture the shared structure in Figure \ref{fig:tr_norm}.
    \item \textbf{What happens if you add MLPs?} Results of adding the MLPs with hidden layer ratio being 4 are shown in Figure \ref{fig:tr_mlp}. We notice seemingly noisier training with both optimizers, and SGD also loosing the plateaus; both runs didn't result in learning the unseen skip -2.
    \item \textbf{Maybe Muon needs more regularization, what happens if you increase the weight decay?} The original run is ran with the default setup of \texttt{SingleDeviceMuon} from \citet{jordan2024muon}, which has weight decay 0. Here we present results if we increase the weight decay to 0.1 and 0.2 in Figure \ref{fig:tr_wd}. We see that the weight decay doesn't help. Again this may be because weight decay constraints both heads equally, and to see the regularization that could produce a similar behavior to SGD, one needs to constrain them unequally, perhaps with a weight decay schedule per head.
    \item \textbf{What happens if you use only one head?} We plot the results in Figure \ref{fig:tr_one}. With one head, SGD run was more unstable, so we lowered the learning rate to 0.1. Muon still doesn't generalize to unseen skip, and perhaps this time the reason lays more closely to the theory on deep linear networks, or it could be due to all the nonlinearities deviating from the theory setting.
    \item \textbf{What happens if the head rank $R$ is smaller than $N$, is SGD still generalizing?} The results with ranks $R=8,4$ suggest that SGD struggles to find generalizable structures in these settings.
    \item \textbf{Can curriculum help Muon learn all skips with shared representations?} We were curious about this, and introduced couple of different curricula to see whether they can recover the effect of gradual learning with SGD. In Figure \ref{fig:curr}, we see that none of our curricula really helped with learning the skip -2 which is unseen in training. Our hypothesis is that Muon's problem with activating all heads at the same time is still causing the problem, where the information for different skips, though highly related, ends up being learned by different heads.
    \item \textbf{Can dynamic curriculum in weight activations help learn all skips with shared representations?} Another type of curriculum is allowing the weights to update gradually through training. The motivation behind this is the observed dynamic curricula in weight activation of GD. We implement weight activation "head-by-head", and within a head, $QK$ circuit is activated "neuron-by-neuron". A part of weights being active means the optimizer update is applied to it from that point onwards. That means that the value matrix for the first head is activated first, and the $Q, K$ matrices rows are activated one by one as training progresses. Once all the rows of the $Q, K$ matrices are active, the second head actives. $V$ is fully active immediately, while $QK$ circuit activates neuron-by-neuron (row-by-row). We show the results when using AdamW optimizer with this dynamic weights curricula in Figure \ref{fig:tr_dyn_curr}, observing it does better than the plain AdamW, learning the unseen skip -2. We note that the success of learning the unseen skip highly varied when we change the learning rate. The use of Muon with the dynamic curricula in weights is not possible because of orthogonalization.
\end{itemize}
\begin{figure}
    \centering
    \includegraphics[width=\linewidth]{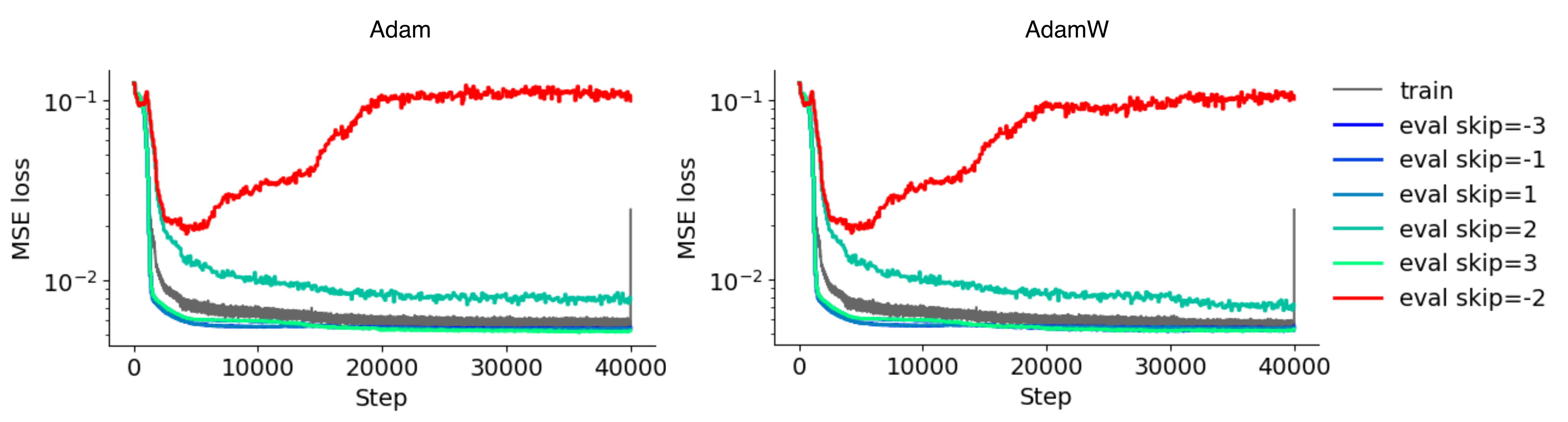}
    \caption{Results on transformer cyclic task for Adam and AdamW optimizers, trained in the same setup of Figure \ref{fig:tr}.}
    \label{fig:tr_adam}
\end{figure}
\begin{figure}
    \centering
    \includegraphics[width=\linewidth]{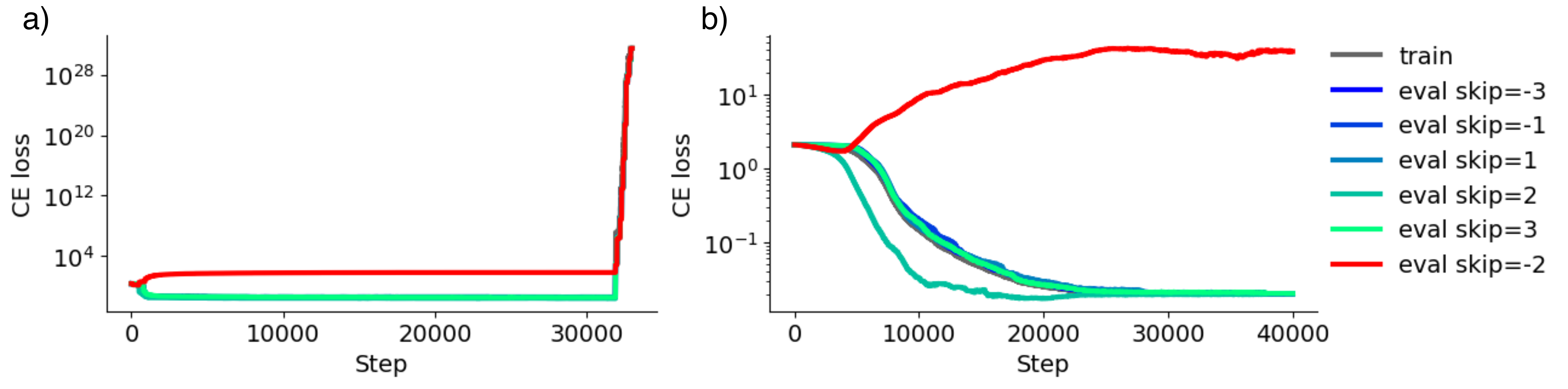}
    \caption{Cross-entropy loss curves for a) SGD and b) Muon in the otherwise same setting of the cycle task as in the main text.}
    \label{fig: tr_ce}
\end{figure}
\begin{figure}
    \centering
    \includegraphics[width=\linewidth]{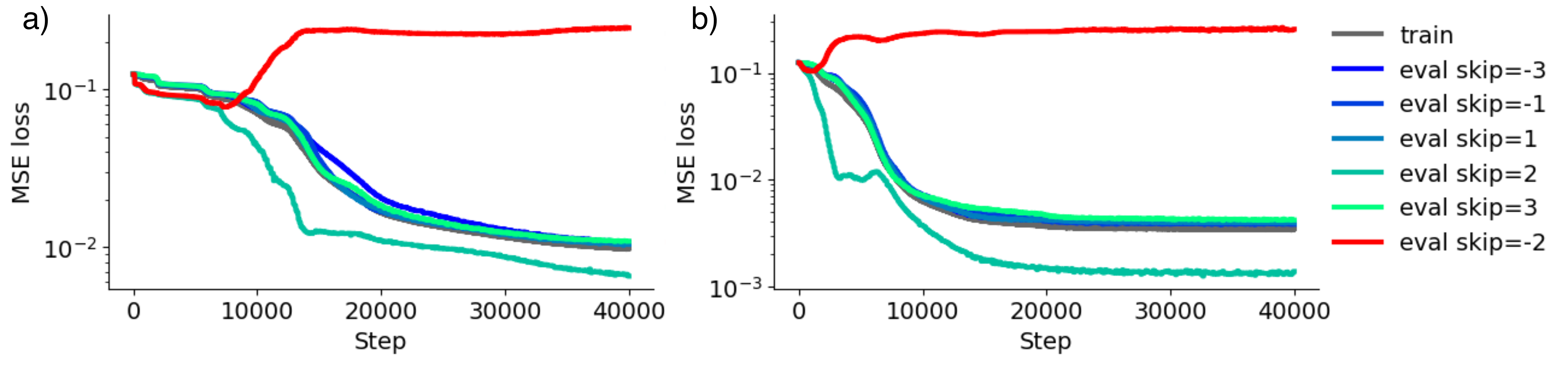}
    \caption{MSE loss curves on transformers with layer norm trained with a) SGD and b) Muon.}
    \label{fig:tr_norm}
\end{figure}
\begin{figure}
    \centering
    \includegraphics[width=\linewidth]{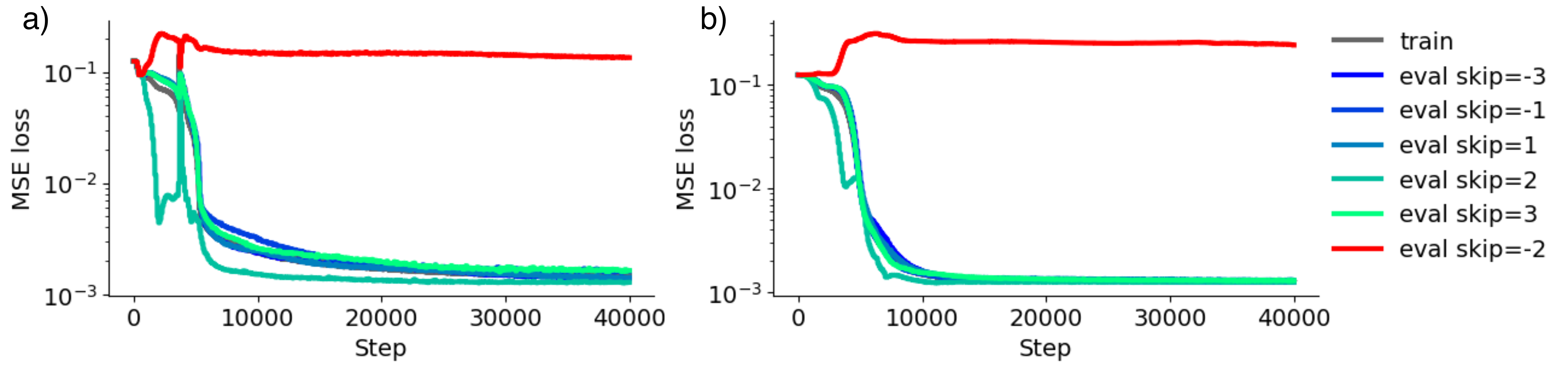}
    \caption{The loss curves shown for a) SGD and b) Muon, when the models now additionally have MLPs after attention layers.}
    \label{fig:tr_mlp}
\end{figure}
\begin{figure}
    \centering
    \includegraphics[width=\linewidth]{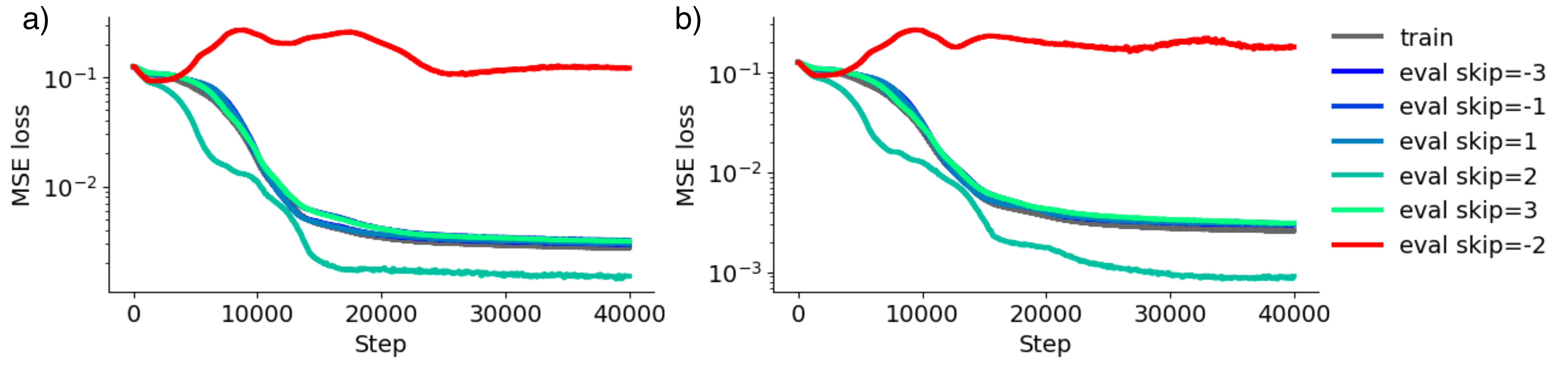}
    \caption{Muon trained transformer on our task with weight decay a) 0.1 and b) 0.2. We see that weight decay did not improve the generalization of the model.}
    \label{fig:tr_wd}
\end{figure}
\begin{figure}
    \centering
    \includegraphics[width=\linewidth]{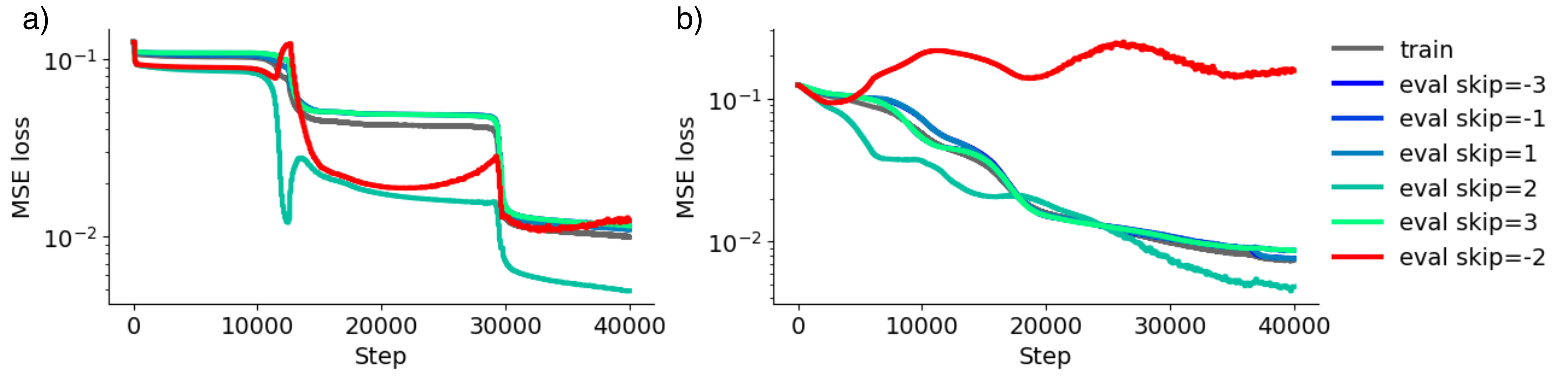}
    \caption{a) SGD and b) Muon trained transformers, now with a single head of dimension 12.}
    \label{fig:tr_one}
\end{figure}
\begin{figure}
    \centering
    \includegraphics[width=\linewidth]{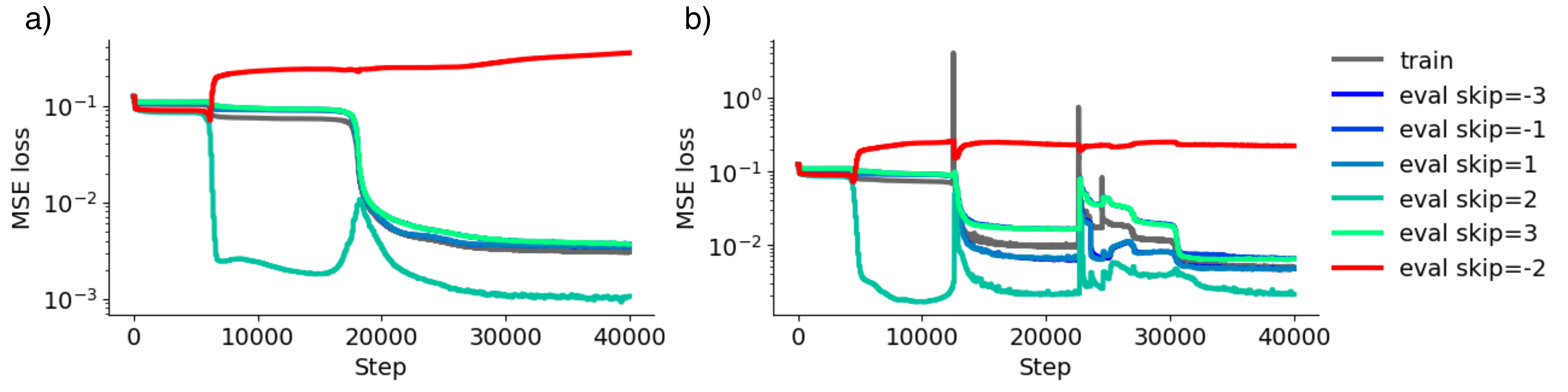}
    \caption{Base setting SGD trained transformer, but instead with ranks a) R=8 and b)R=4 per head. Having smaller rank took away the generalization. We also notice more training instabilities with $R=4$.}
    \label{fig:tr_lr}
\end{figure}
\begin{figure}
    \centering
    \includegraphics[width=\linewidth]{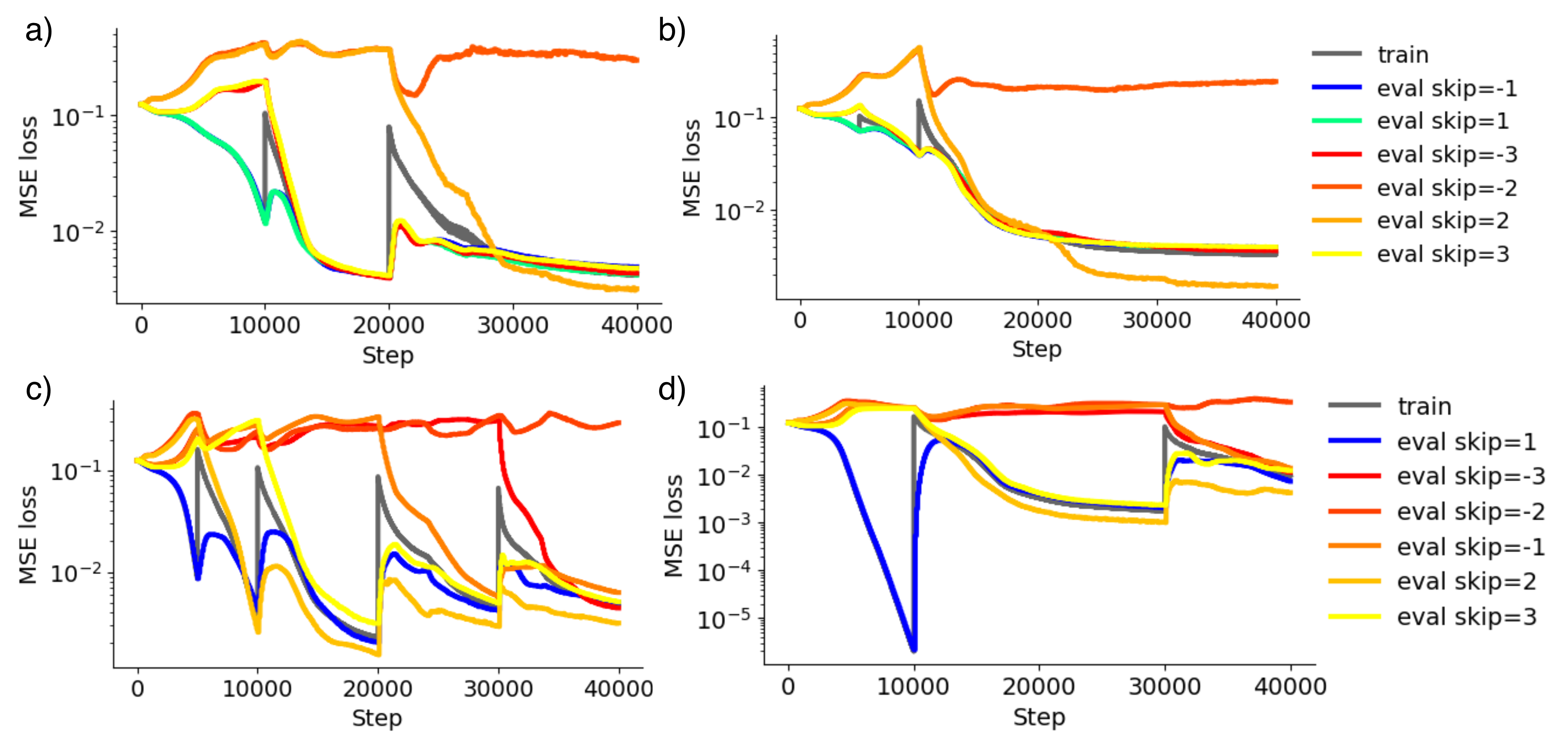}
    \caption{Introducing a couple of different curricula into the task of learning cycles with Muon. Curricula are a) at step 0: skips 1 and -1, at step 10000: skips 3 and -3, at step 20000: skip 2; b) at step 0: skips 1 and -1, at step 5000: skips 3 and -3, at step 10000: skip 2 c) at step 0: skip 1, at step 5000: skip 2, at step 10000: skip 3, at step 20000: skip -1, at step 30000: skip -3; d) at step 0: skip 1, at step 10000: skips 2, 3, at step 30000: skips -1 and -3.}
    \label{fig:curr}
\end{figure}
\begin{figure}
    \centering
    \includegraphics[width=0.7\linewidth]{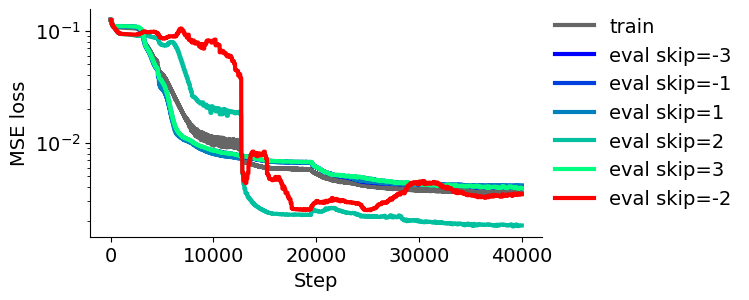}
    \caption{Employing the dynamic curriculum in weight activation, using AdamW optimizer with learning rate 0.0004. We see that learning the unseen skip is more successful now, compared to vanilla AdamW run.}
    \label{fig:tr_dyn_curr}
\end{figure}


\end{document}